\documentclass[sigconf,natbib=true,anonymous=false]{acmart}

\AtBeginDocument{%
  \providecommand\BibTeX{{%
    \normalfont B\kern-0.5em{\scshape i\kern-0.25em b}\kern-0.8em\TeX}}}

\setcopyright{acmlicensed}
\copyrightyear{2021} 
\acmYear{2021} 
\setcopyright{acmcopyright}\acmConference[CIKM '21]{Proceedings of the 30th ACM
International Conference on Information and Knowledge Management}{November
1--5, 2021}{Virtual Event, QLD, Australia}
\acmBooktitle{Proceedings of the 30th ACM International Conference on Information
and Knowledge Management (CIKM '21), November 1--5, 2021, Virtual Event, QLD,
Australia}
\acmPrice{15.00}
\acmDOI{10.1145/3459637.3482343}
\acmISBN{978-1-4503-8446-9/21/11}
\usepackage{booktabs} 
\usepackage{multirow}

\usepackage{arydshln, booktabs}
\usepackage{cellspace}

\usepackage{tikz}
\usetikzlibrary{arrows,positioning, calc}
\tikzstyle{vertex}=[draw,fill=black!15,minimum size=20pt,inner sep=0pt]
\usepackage{float}

\usepackage{amsmath}
\usepackage{amsthm}

\usepackage{amssymb}
\usepackage{xcolor}

\usepackage[ruled,vlined]{algorithm2e}
\usepackage{balance}

\def\ouralgo{{\text{WEQ}}}
\def\R{{\mathbb R}}
\def\P{{\mathbb P}}
\def\inf{{M}}
\def\rank{{\text{sr}}}
\def\negs{{\kappa}}

\begin{document}
\fancyhead{}
\title{Fast Extraction of Word Embedding from Q-contexts}


\author{Junsheng Kong}
\authornote{Equal contribution. Work was done during internship at
Tencent Quantum Lab.}
\authornote{Also with The Key Laboratory of Big Data and Intelligent
Robot (South China University of Technology), Ministry of Education.}
 \affiliation{
   \institution{School of Software Engineering, South China University of Technology}
  \country{China}
  }
\email{sescut_kongjunsheng@mail.scut.edu.cn	}
  
 \author{Weizhao Li}
 \authornotemark[1]
  \authornotemark[2]
 \affiliation{
   \institution{School of Software Engineering, South China University of Technology}
  \country{China}
  }
 \email{se_weizhao.li@mail.scut.edu.cn}
 
 \author{Zeyi Liu}
 \affiliation{
   \institution{University of Cambridge}
  \country{United Kingdom	}
  }
  \email{zl411@cam.ac.uk}
  
 \author{Ben Liao}
  \authornote{Corresponding Author.}
 \affiliation{
   \institution{Tencent Quantum Lab}
  \country{China}
  }
  \email{bliao@tencent.com}
  
\author{Jiezhong Qiu}
 \affiliation{
   \institution{Tsinghua University}
  \country{China}
  }
  \email{qiujz16@mails.tsinghua.edu.cn}
 
 \author{Chang-Yu Hsieh}
 \affiliation{
   \institution{Tencent Quantum Lab}
  \country{China}
  }
\email{kimhsieh@tencent.com}

 \author{Yi Cai}
  \authornotemark[2]
  \authornotemark[3]
 \affiliation{
   \institution{School of Software Engineering, South China University of Technology}
  \country{China}
  }
  \email{ycai@scut.edu.cn}
  
\author{Shengyu Zhang}
\authornotemark[3]
 \affiliation{
   \institution{Tencent Quantum Lab}
  \country{China}
  }
  \email{shengyzhang@tencent.com}

\renewcommand{\shortauthors}{Junsheng Kong, et al.}

\begin{abstract}
  The notion of word embedding plays a fundamental role in natural language processing (NLP). However, pre-training word embedding for very large-scale vocabulary is computationally challenging for most existing methods. In this work, we show that with merely {\it a small fraction of contexts (Q-contexts)} which are {\it typical} in the whole corpus (and their mutual information with words), one can construct high-quality word embedding with negligible errors.
  Mutual information between contexts and words can be encoded canonically as a sampling state, thus, Q-contexts can be fast constructed.
  Furthermore, we present an efficient and effective $\ouralgo$~method, which is capable of extracting word embedding {\it directly} from these typical contexts. In practical scenarios, our algorithm runs 11$\sim$ 13 times faster than well-established methods. By comparing with well-known methods such as matrix factorization, word2vec, GloVe and fasttext, we demonstrate that our method achieves comparable performance on a variety of downstream NLP tasks, and in the meanwhile maintains run-time and resource advantages over all these baselines.
\end{abstract}



\maketitle

\section{Introduction}
Word embedding plays a fundamental role in the development and real-world applications of natural language processing (NLP). It efficiently provides meaningful representations of individual words in a continuous space, allowing smooth integration with machine learning models in various downstream NLP tasks~\citep{maruf-haffari-2018-document, khabiri2019industry}. The notion of word embedding is also the predecessor of follow-up
deep contextualization models,
including the recently discovered powerful pre-trained contextual embedding models such as ELMo~\cite{Elmo} and BERT~\cite{devlin2019bert}.

High-quality embedding of words can help boost the performance of many machine learning models in NLP tasks.
Recent work about word embedding can be categorized into two genres, i.e., neural network based methods~\cite{DBLP:conf/nips/MikolovSCCD13, fasttext,DBLP:conf/emnlp/PenningtonSM14} and global matrix factorization based methods~\cite{deerwester1990indexing,NIPS2014_feab05aa,arora}. Word2Vec, GloVe and fasttext are the most popular neural network based methods.
Most of the existing methods focus on improving the performance of word embedding.
However, it is computationally expensive to obtain such word embedding --- it takes several days and, typically, around a hundred CPU cores to attain decent quality representation of words
\cite{DBLP:conf/nips/MikolovSCCD13, Mnih, DBLP:conf/emnlp/PenningtonSM14, mikolov13}.
Global matrix factorization based methods for generating word embedding have roots stretching as far back as LSA\cite{deerwester1990indexing}. These methods utilize low-rank approximations to decompose large  matrices  that capture statistical information about a corpus.
Previous work has shown that both the word2vec and GloVe methods can be viewed as implicit factorization of special information matrices\cite{NIPS2014_feab05aa, arora}.
Although global matrix factorization based method is more efficient than the neural network based methods, it still needs to factorize a large $n \times n$ information matrix for large-scale vocabulary, where $n$ is the size of vocabulary.
This makes it highly expensive to directly factorize and calculate for large-scale word embedding learning.

To address the efﬁciency limitations of current work, we propose to study word embedding learning for large-scale vocabulary with the goal of efficiency and theoretical guarantees.
Recent literature has shown that quantum perspective can thus provide advantages for classical machine learning~\cite{rebentrost2014quantum, zeng2016quantum}. Coecke et al. have previously demonstrated a potential quantum advantage for NLP in various ways including by algorithmic speed-up for search-related or classification tasks~\citep{coecke2020foundations}. 
By mimicking how word-meanings are encoded in quantum states, we design our algorithms implemented on classical computers to speed up the word embedding learning problem.
The main idea is to construct a small and typical information matrix which is a good approximation of the original information matrix. 
Both the construction and the factorization of the small matrix require a low cost.
With this design, we are able to demonstrate running-time supremacy for solving a large-scale word embedding problem and maintain accuracy for various downstream NLP tasks.

We reveal a simple relation between a word vector $e_w$ for a target word $w$ and what we call \textit{Q-contexts}.
Q-contexts are a small fraction of contexts that are typical in the whole corpus, capable of capturing the most important information encoded in the information matrix $M$ -- they are certain rows of $M$ chosen to represent the original information matrix. 
The word vector of $w$ is shown to be a combination of its interaction with these contextual environments
\begin{equation}\label{main-eq}
e_w\approx\sum_{c\in \text{ Q-contexts}} \lambda_c M_{c,w}    
\end{equation}
where $\inf_{c,w}$ is the entry in the information matrix $M$ indexed by $c$ and $w$, $\lambda_c$ is a constant vector for a context $c$ to be determined
(for a more detailed description see Section~\ref{Definition}).
Information matrix $M$ can be naturally encoded as a sampling state, enabling a fast construction of Q-contexts.
To the best of our knowledge, it has not been studied to extract meaningful word embedding from its mutual information with a few contexts.

Based on these, we develop a \ouralgo~method that substantially accelerates the word embedding learning
--- in fact, our method is at least $11\sim 13$ times faster than well-established methods, and has fewer resource requirements on CC corpus.
We show empirically that \ouralgo~achieves comparable performance in comparison to well-known methods, such as the direct matrix factorization, word2vec, GloVe and fasttext \cite{DBLP:conf/nips/MikolovSCCD13,DBLP:conf/emnlp/PenningtonSM14, NIPS2014_feab05aa, fasttext} and maintains high accuracy in various downstream NLP tasks.
Further, \ouralgo’s efficiency and effectiveness are theoretically backed up. The small Q-contexts matrix is a good approximation of the original information matrix with negligible error, maintaining the representation power of its learned embedding.

The organization of the rest of this article is as follows. In Section~\ref{PRELIMINARIES-section}, 
we recall a general matrix factorization perspective on well-known word embedding methods. In Section~\ref{method-section}, we introduce the $\ouralgo$~method. In Section~\ref{proof-section}, we analyze the approximation error of the Q-contexts with theoretical proof. In Section~\ref{experiments-section}, we conduct a comprehensive set of experiments demonstrating the accuracy and efficiency of our method. In Section~\ref{related word-section}, we review the related work of word embedding. Finally, we give a conclusion in Section~\ref{sec:conclusion}.

\begin{table}[htbp]
\caption{Notation.}
\centering
\setlength{\tabcolsep}{1.8mm}{
\begin{tabular}{c|l}
\toprule
\textbf{Notation}                & \textbf{Description}                \\ \midrule
$P$            & the multiset of context-word pairs           \\
$w$            & the target word          \\
$c$            & context word around the target word          \\
$\#(c,w)$            & the number of co-occurrences of $c$ and $w$ in $P$          \\
$\#(c)$            & the number of times of $c$ appears in $P$          \\
$\#(w)$            & the number of times of $w$ appears in $P$          \\
$|P|$            & $\sum_c \sum_w \#(c, w)$           \\
$\inf_{c,w}$            & entry in information matrix indexed by $c$ and $w$         \\
$e_w$            & word (row) vector for the target word $w$          \\
$E_w$            & $E_w=\begin{pmatrix}e_{w_1}\\\vdots\\e_{w_n}\end{pmatrix}$          \\
$e_c$            & context (row) vector for context $c$          \\
$\|\cdot \|$            & the $\ell^2$-norm of a vector           \\ 
$A$                & matrix            \\ 
$\|A\|_F$             & the Hilbert-Schmidt norm of $A$          \\
$\|A \|_{op}$            & the operator norm of  matrix $A$           \\ 
$A_{i,*}$            & row $i$ of $A$          \\
$A_{*,j}$            & column $j$ of $A$          \\
$A^\top$            & the transpose of $A$          \\
$\text{nnz}(A)$            & the number of nonzeros in $A$          \\
$R$            & Q-contexts matrix          \\
$\tilde R$            & the normalized version of $R$         \\
\bottomrule
\end{tabular}
}
\label{tab:notation}
\end{table}

\section{PRELIMINARIES}
\label{PRELIMINARIES-section}

Commonly, the problem of word embedding is learned by capturing the semantic relationship between word-context pairs $(w,c)$. For a target word $w$, its context word c is obtained from the neighborhood centering around the locations where $w$ appears in the corpora. 
Previously established results show that the factorization of information matrices provides a united framework for many important existing word embedding algorithms, including word2vec, GloVe, PMI, and NCE \cite{arora,DBLP:conf/emnlp/PenningtonSM14,Mnih,NIPS2014_feab05aa}.
The main difference between these methods lies in different choices of mutual information matrix $M_{c,w}$ between contexts and words.

As indicated in 
\citep{church,arora,dagan,turney},
 factorizing the following point-wise mutual information matrix (PMI) yields effective word representations $M_{c,w}=\log \frac{\#(c,w)|P|}{\#(c)\#(w)}$, where $P$ is the multiset of context-word pairs, \#(c,w) is the number of co-occurrences of context word $c$ and target word $w$ in the $P$, $\#(c)$ and $\#(w)$ are the number of times $c$ and $w$ appear in $P$ respectively. 
It is shown in \cite{NIPS2014_feab05aa} that 
word vectors from word2vec can be obtained from factorization of a shifted version of PMI:
$M_{c,w}=\log \frac{\#(c,w)|P|}{\#(c)\#(w)\cdot \negs},$
where $\negs$ denotes the number of negative samples. They also show that NCE model \cite{Mnih}
is in fact factorizing $M_{c,w}=\log \frac{\#(c,w)}{\#(c)\cdot \negs}.$
To improve performance, a positive version of PMI (PPMI)
$M_{c,w}~=~\log_+\frac{\#(c,w)|P|}{\#(c)\#(w)} $
and a shifted version of PPMI (SPPMI with shift parameter $\negs$) 
$M_{c,w}~=~\log_+\frac{\#(c,w)|P|}{\#(c)\#(w)\cdot \negs}$
are proposed,
where $\log_+(x)~=~\max(\log x,0).$

It is shown \cite{arora} that GloVe objective is in fact optimizing (modulo some error term)
\begin{equation*}
    \sum \#(c,w)\Big (\log \#(c,w)-e_{c}\cdot e_w-\|e_{c}\|^2-\|e_w\|^2\Big )^2 ,
\end{equation*}
where $e_c$ is the context vector for context $c$ and $e_w$ is the word vector for the target word $w$.
In their theory, the authors also show that for some constant $Z$:
\begin{align*}
    \log p(c,w)&\approx\|e_{c}+e_w\|^2/2d-2\log Z \\
    \log p(w)&\approx\|e_w\|^2/2d-\log Z .
\end{align*}
Since $p(w)\approx \frac{\#(w)}{|P|}$, we conclude 
$$e_{c}\cdot e_w\approx \log \left[ (|P|Z)^{4d}\cdot \frac{\#(c,w)}{(\# (c)\#(w))^{2d}}\right].$$

Factorization of mutual information matrices constitutes a unified framework of these word embedding algorithms: $M_{c,w}~=~e_c\cdot e_w$. We list necessary notations and their descriptions in Table~\ref{tab:notation}.

\section{Method}
\label{method-section}

In this section, we present $\ouralgo$ method which is an efficient and effective method for large-scale word embedding learning problem. We develop the $\ouralgo$ method to construct and factorize a small typical information matrix that approximates the original information matrix. The $\ouralgo$ method is composed of three steps, as illustrated in Fig.~\ref{fig:method}. 
First, it calculates the information matrix $M$ from co-occurrence matrix $X$. Secondly, it constructs the small typical information matrix (Q-contexts) from the original information matrix through $\ell^2$-norm sampling. Third, it conducts the singular value decomposition of Q-contexts matrix to obtain the word embedding.

\subsection{Q-contexts Definition}
\label{Definition}
We first introduce the $\ell^2$-norm sampling, then describe the definition of Q-contexts.

\textbf{$\ell^2$-norm sampling}: The $\ell^2$-norm sampling technique has well exhibited its effectiveness in machine learning \citep{DBLP:conf/nips/HazanKS11, DBLP:conf/nips/SongWZ16} and randomized linear algebra \cite{DBLP:journals/siammax/DrineasMM08}. In fact, the work by Frieze, Kannan, and Vempala \cite{fkv} shows that with certain $\ell^2$-norm sampling assumptions, a form of singular value estimation can be achieved in time independent of the size of input matrix.
Further, the work by Tang~\cite{Ewin} shows that sampling from the projection of a vector onto a subspace is not outside the realm of feasibility. 

Inspired by these, we leverage the $\ell^2$-norm sampling to construct a small typical information matrix to solve the large-scale word embedding learning problem.
We now elaborate on our proposed relation between Q-contexts and words in Equation~\eqref{main-eq}.

Given an information matrix $M$, we first encode the information matrix $M$ into an $\ell^2$-norm state which will be described in detail in Section~\ref{WEQ-section}.
Now that the mutual information $M$ is prepared to be a state, each $\ell^2$-norm sampling yields a context $c_i$ with probability $p_{c_i}$.
The collection of the corresponding row vectors $r_{c_i} = M_{c_i,*}$ is what we call Q-contexts.
In our proposed scheme, $p_c$ is designed to correctly reflect the amount of information carried by the context $c$, so that contexts with more information are more likely to be sampled.

\begin{figure}[t]
    \begin{center}
    \includegraphics[scale=0.3]{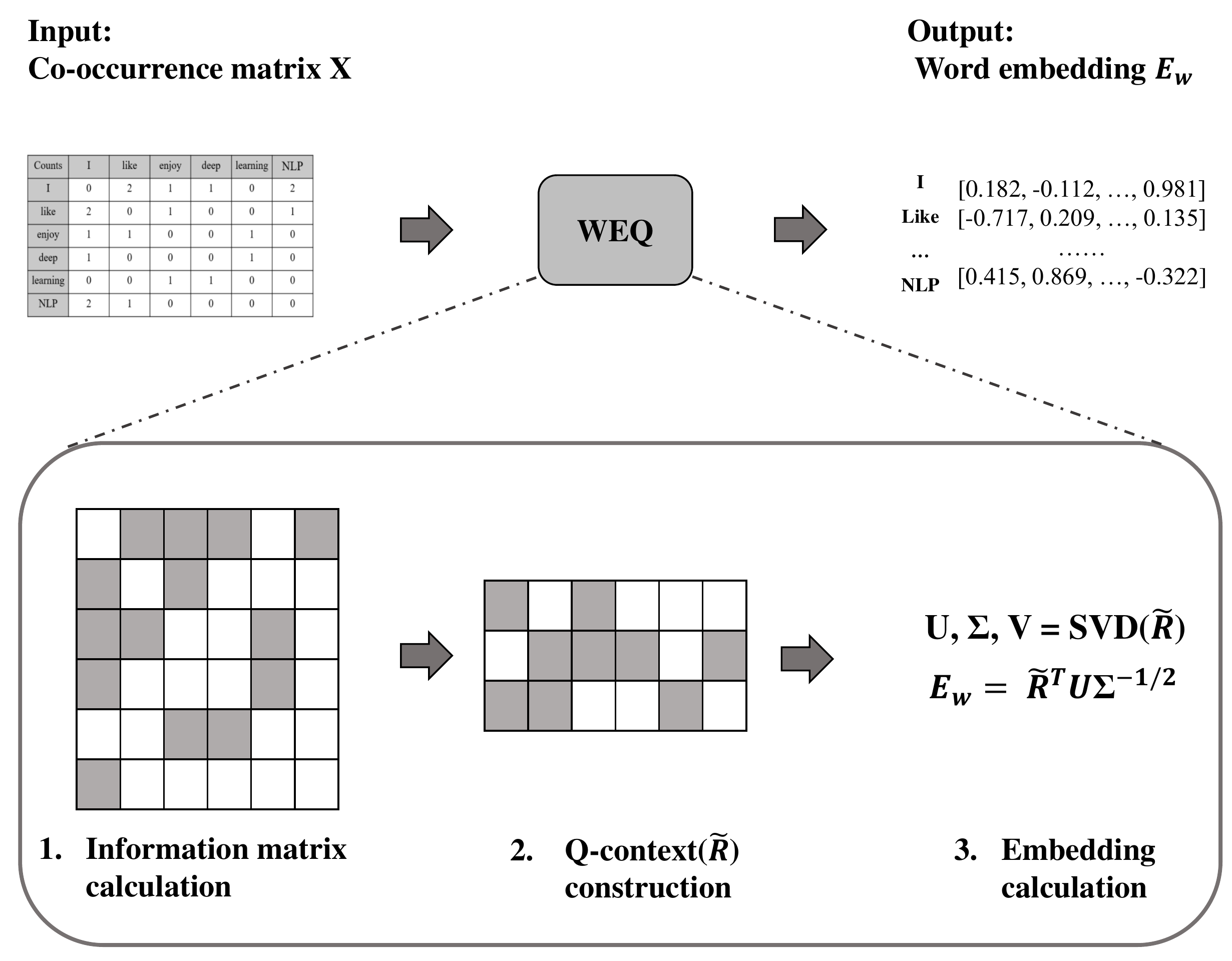}
    \caption{\ouralgo~method. }
    \label{fig:method}
    \end{center}
\end{figure}

\begin{definition}
A Q-contexts matrix $R\in\R^{k\times n}$ is the collection of $k$ rows $r_{c_i}$ in the information matrix $M$:
$$R~=~\begin{pmatrix}r_{c_1}\\\vdots\\r_{c_k}\end{pmatrix}$$
where context $c_i$ is the $i$-th sampling outcome from the information matrix $M$.
\end{definition}

The central idea in Equation~\eqref{main-eq} is that a few $\ell^2$-norm sampling of the information matrix provides sufficient information 
such that a good
word embedding can be obtained from a linear combination of its mutual information with Q-contexts
$$e_w\approx\sum_{c \text{ runs over row indices in } R} r_{c,w}\lambda_c$$
where 
the \textit{context coefficients vectors} (i.e., $\lambda_c$'s) are to be determined in Section~\ref{WEQ-section}.

\subsection{Word Embedding from Q-contexts ($\ouralgo$)}\label{WEQ-section}
We now give the full algorithm (Algorithm~\ref{main-algo}) for computing word embedding based on Q-contexts matrix $R$. {\ouralgo} method consists of three steps: information matrix calculation, Q-contexts construction, and calculating word embedding from Q-contexts.

\textbf{Step 1: Information matrix calculation.}
As we have discussed in Section~\ref{PRELIMINARIES-section}, there are different kinds of information matrices. Here, we choose PPMI and SPPMI matrices proposed in  \cite{NIPS2014_feab05aa} which shows that exact factorizing (PPMI / SPPMI) matrix with SVD is at least as good as SGNS’s solutions. Given a co-occurrence matrix, we construct the PPMI information matrix as follows:
$$M_{c,w}(PPMI)~=~\log_+\frac{\#(c,w)|P|}{\#(c)\#(w)} $$
and a shifted version of PPMI (SPPMI with shift parameter $\negs$) 
$$M_{c,w}(SPPMI)~=~\log_+\frac{\#(c,w)|P|}{\#(c)\#(w)\cdot \negs}.$$

\textbf{Step 2: Q-contexts construction.}
As in Definition 3.1, we will encode the information matrix $M$ into an $\ell^2$-norm state, which admits our fast construction of Q-contexts matrix $R$. 

In our method, we propose a data structure that achieves fast $\ell^2$-norm sampling in practice. 
For the information matrix $M\in \mathbb{R}^{n \times n}$, we store the cumulative summation of its row-squared $(\|r_1\|^2, \|r_1\|^2+\|r_2\|^2, \cdots, \|M\|_F^2)$. To perform the  $\ell^2$-norm sampling, we first generate a random number $a$ from the uniform distribution $\mathcal{U}(0,\|M\|_F^2)$ and perform an efficient binary search algorithm to find the leftmost index such that $a$ is less than or equal to the corresponding cumulative sum. Through repeating the $\ell^2$-norm sampling for $k$ times, we get the small Q-contexts $R$. We normalize each row to get $\tilde R$. 

In addition, we can use column sampling to reduce the matrix $\tilde R$ again.
Applying
the theorem and similar analysis to $\tilde R^\top$, we get a smaller matrix $\tilde C$ for which with high probability
$$\tilde R\tilde R^\top \approx \tilde C \tilde C^\top.$$
Then
$$U\Sigma^2U^\top \approx \tilde R\tilde R^\top \approx \tilde C \tilde C^\top.$$
We obtain good approximations of the right singular vectors $U$ and singular values $\Sigma$ of $\tilde R$, by simply doing the same calculations for the much smaller matrix $\tilde C$. 

\begin{algorithm}[h]
    \SetAlgoLined
   \caption{Q-contexts construction}
   \label{q-contexts-algo}
    \LinesNumbered
    \SetKwInOut{Input}{input}
    \SetKwInOut{Output}{output}
    \Input{information matrix $M\in \mathbb{R}^{n \times n}$; \\
   number of samples $k$}
   \Output{Q-contexts matrix $\tilde R \in \mathbb{R}^{k \times n}$}
   /* Prepare the state: compute the cumulative sum of M*/ 
   $S(M)=(\|r_1\|^2 , \| r_1\|^2+\|r_2\|^2, \|r_1\|^2+\|r_2\|^2+\|r_3\|^2,\cdots, \|M\|_F^2)$\;
   \For{$i\leftarrow 1$ \KwTo $k$}{
        /* sample a row $i$ from the state of $M$ */ \\
        Generate $s\sim \mathcal{U}(0,\|M\|_F^2)$\;
        Search $i$ such that $S(M)_{i-1}\leq s<S(M)_i$\;
        /* normalization */ \\
        Form $\tilde R_{i,*} = \frac{\Vert M \Vert_F}{\sqrt{k}\|M_{r_i,*} \|}M_{r_i,*}$\;
   }
   Return $\tilde R$\;
\end{algorithm}

\textbf{Step 3: Calculating word embedding from Q-contexts.}
It is known \cite{Mnih} that embedding matrix
$E_w$ taken to be
the form $E_w=V_w\sqrt \Sigma$ can be beneficial in predictive performance,
where $V_w$ and $\Sigma$ are right singular vectors and singular values of information matrix $M$. In view of Theorem~\eqref{main-thm} later, with high probability $\tilde R^\top \tilde R\approx M^\top M$, 
which implies that $\tilde R\approx U\Sigma V_w^\top$ where $U$ is the left singular vectors of $\tilde R$. Since $U$ is a (partial-)isometry, we obtain the matrix form of Equation~\eqref{main-eq}:
\begin{equation}\label{main-eq-matrix}
    E_w= V_w\sqrt \Sigma\approx \tilde R^\top U\Sigma^{-1/2}= R^\top \Lambda,
\end{equation}
where $\Lambda = D^{-1}U\Sigma^{-1/2}$ is the matrix of context coefficients.

\textbf{Complexity Analysis.} 
We get Algorithm~\ref{main-algo} by putting the above procedures together. 
As for line 1, it requires $\mathcal O(\text{nnz}(X))$ time to perform point-by-point operations on the nonzero elements of the co-occurrence matrix $X$. 
As for line 2 and line 3, time complexity for state preparation is $\mathcal O(n)$ and for $\ell^2$-norm sampling is $\mathcal{O} (k \log n)$.
As for line 4, $\mathcal O(\text{poly}(k))$ time is spent in singular vectors computation. In MF method, it needs to take highly expensive $\mathcal O(\text{poly}(n))$ time to compute singular vectors of original information matrix $M$.
As for line 5, $\mathcal O(kdn)$ time is spent in matrix multiplication.

\begin{algorithm}[h]
    \SetAlgoLined
   \caption{$\ouralgo$ method}
   \label{main-algo}
    \LinesNumbered
    \SetKwInOut{Input}{input}
    \SetKwInOut{Output}{output}
    \Input{sparse co-occurrence matrix $X\in \mathbb{R}^{n \times n}$; \\
   number of samples $k$, and embedding dimension $d$}
   \Output{embedding matrix $E_w$}
   Compute information matrix $M$ (e.g. PPMI, SPPMI) from $X$ \;
   Construct the Q-contexts matrix $\tilde R$\ according to Algorithm~\ref{q-contexts-algo}\;
   (Optional) Construct the much smaller Q-contexts matrix $\tilde C^\top$\ from $\tilde R^\top$\ according to Algorithm~\ref{q-contexts-algo} again\;
   Compute the top $d$ left singular vectors $U$ and singular values $\Sigma$ for the Q-contexts matrix $\tilde R$\ or $\tilde C$\;
   Return $E_w=\tilde R^\top U \Sigma^{-1/2}$\;
\end{algorithm}

\section{Theoretical proof}
\label{proof-section}

In this section, we demonstrate that very few $\ell^2$-norm samplings on the information matrix $M$ suffice to extract high-quality word embeddings. 

The idea that word embedding hidden in mutual information matrix $M$ can be extracted from a few $\ell^2$-norm sampling on the state of $M$ might seem surprising at first glance. The root of this phenomenon stems from the fact that Q-context $r_c$ is a randomized object, its probabilistic behavior has an almost deterministic nature in the sense of Central Limit Theorem.

In more precise terms, word representation hides in the symmetric form $M^\top M.$
In view of the random nature of context $r_c$,
this form can be written as the expectation of the rank-one matrix $S=\frac{1}{p_c}r_c^\top r_c$ relating to context $c$,
\begin{equation}\label{sum-h}
    M^\top M=r_1^\top r_1+\cdots+r_n^\top r_n=\sum_{c} p_c\times \frac{1}{p_c}r_c^\top r_c,
\end{equation}
where $p_c$ is the probability of getting context $c$ from a measurement, and $c$ runs over the set of contexts.
Precisely, $S$ is the random matrix taking value $\frac{1}{p_c}r_c^\top r_c$ with probability $p_c$.

The idea of Central Limit Theorem and its general form of probabilistic measure concentration is also valid 
in a functional analytic (matrix) context, where a scalar-valued random variable is generalized to a matrix-valued one.
One of the most celebrated theorems is the operator/matrix Bernstein concentration inequality \cite{tropp}.

\begin{theorem}\label{berstein}
Let $S_i\in\R^{n\times n},i=1,\cdots,k$ be a sequence of independent identically distributed symmetric norm-bounded matrices with mean 
$\mu$.
Then for sufficiently small $\epsilon,$ we have
$$\P\left[\left\|\frac{1}{k}\sum S_i-\mu\right\|_{op}\geq \epsilon\right]\leq 4\cdot  \rank(\varsigma)\cdot     \exp{\left(-\frac{k\epsilon^2}{2\|\varsigma\|_{op}^2}\right)},$$
where $\varsigma^2$ is the 
covariance matrix of $S_1$: $\varsigma^2~=~\mathbb{E}(S_1-\mu)^2$, and $\rank(\cdot)=\frac{\|\cdot\|_F^2}{\|\cdot\|_{op}^2}$ denotes the stable rank of a matrix.
\end{theorem}

Thanks to the matrix concentration Theorem~\ref{berstein} , we see that, with high probability, the expectation $\mu =\mathbb{E} S=M^\top M$ can be well approximated by the average of 
$k~\geq~ \frac{3\|\varsigma\|_{op}^2}{\epsilon^2}\log \rank(\varsigma)$ samples $S_1,\cdots,S_k$.
The average of these samples is 
\begin{equation*}
    M^\top M\approx \frac{1}{k}\left(S_1+\cdots+S_k\right)=\tilde R^\top \tilde R,
\end{equation*}
where $\tilde R$ is the normalized version of Q-contexts matrix $R$: 
\begin{equation*}
    \tilde R=D^{-1}R, D=\text{diag}(\sqrt{kp_{i_1}},\cdots,\sqrt{kp_{i_k}}).
\end{equation*}
Precisely, we get the following theorem:
\begin{theorem}\label{main-thm}
Let $p_c=\frac{\|r_c\|^2}{\|M\|_F^2}$, then we have
$$\P\left[\left\|\tilde R^\top \tilde R-M^\top M\right\|_{op}\geq \epsilon\right]\leq 4\cdot \rank(\varsigma)\cdot \exp{\left(-\frac{k\epsilon^2 }{2\|\varsigma\|_{op}^2}\right)},$$
where $\varsigma = \sqrt{\mathbb{E} (S-\mu)^2}$ with stable rank $\rank(\varsigma)$ being bounded by the rank of $M$, and
\begin{equation}\label{sigma-inequality-theorem}
    \|\varsigma\|_{op}\leq\min\left(\|M\|_F\|M\|_{op}, \sqrt{\|M\|_F^4-\|M^\top M\|_F^2}\right).
\end{equation}
\end{theorem}

\begin{proof}
The probability inequality in the theorem is a direct consequence of Theorem~1. We only need to prove   inequality~\eqref{sigma-inequality-theorem} and the bound on $\rank(\varsigma)$ in the theorem.

By direct calculation, we have
$$\varsigma^2 = \mathbb{E} S^2-\mu^2=\sum_c \frac{1}{p_c}\|r_c\|^2r_c^\top r_c - (M^\top M)^2.$$

Therefore,
\begin{align*}
    \|\varsigma\|_F^2&=Tr(\varsigma^2)=\sum_c\frac{\|r_c\|^4}{p_c}-\|M^\top M\|^2_F\\
    \sum_c p_c&=1.
\end{align*}
The Lagrange multiplier Theorem implies that $\|\varsigma\|_F^2$ achieves minimum only when $p_c^2$ is proportional to $\|r_c\|^4,$ namely
$p_c=\frac{\|r_c\|^2}{\|M\|_F^2}$. In this case, 
$$\|\varsigma\|_{op}^2\leq\|\varsigma\|_F^2=\|M\|_F^4-\|M^\top M\|^2_F$$
which is the first part of inequality~\eqref{sigma-inequality-theorem}.

For the second part, when $p_c=\frac{\|r_c\|^2}{\|M\|_F^2}$ we see that as matrices
\begin{equation}\label{inequality-sig-sq}
    \varsigma^2\leq \sum \frac{1}{p_c}\|r_c\|^2r_c^\top r_c=\|M\|_F^2M^\top M.
\end{equation}
Inequality~\eqref{inequality-sig-sq} implies
$$\|\varsigma\|_{op}\leq \|M\|_F\|M\|_{op}.$$

The stable rank of $\varsigma$ is bounded as a consequence of inequality~\eqref{inequality-sig-sq}:
\begin{align*}
   \rank{(\varsigma)} &\leq rank(\varsigma)=rank(\varsigma^2)\\
   &\leq rank (M^\top M)=rank(M).
\end{align*}




\end{proof}

From the proof of Theorem~\ref{main-thm} (especially the proof of the first part of inequality~\eqref{sigma-inequality-theorem}), 
we can see that the choice of $p_c$ is motivated by the variance minimization scheme. Therefore such a sampling strategy can lead to the matrix mean at a faster speed.

\subsection{Error Analysis}\label{error-analysis}
We remark that that approximation in Theorem~\ref{main-thm} is of high quality, both in theory and practice.

From a mathematical point of view, approximation error in Theorem~\ref{main-thm} is small provided that 
$k~\geq~ \frac{3\|\varsigma\|_{op }^2}{\epsilon^2}\log \rank(\varsigma)$ rows are sampled from the information matrix $M$. Our analysis is tighter than \cite{fkv} since we make use of a stronger bound on matrix concentration (Theorem~\ref{berstein}) concerning stable rank and spectral norm of a matrix, instead of dimension and Frobenius norm.
Stable rank is upper bounded by the rank of a matrix, which is more effective in a low-rank matrix regime such as word embedding problem. 
Spectral norm is bounded above by Frobenius norm as well. Moreover, the Frobenius norm can fail to capture the genuine behavior of random matrices in even very simple examples \cite{tropp}. Therefore, our analysis shows that with very few samples of rows, $\tilde R^\top  \tilde R$
can be a good approximation of $M^\top M$ with negligible error.

These theoretical observations are also consistent with our experiments. Indeed, in Section~\ref{ablation-section}, we will see that with information matrix dimensions varying from 50k to 400k, the number of row samples required to achieve a good approximation is relatively stable --- in our test it ranges from 40k to 50k, while less than only $2\%$ of accuracy is lost compared to direct calculation with the original information matrix.

\section{Experiments}\label{experiments-section}
In this section, we evaluate the proposed $\ouralgo$ method on three different evaluation tasks: Word Similarity, Text Classification and Named Entity Recognition (NER), which have been commonly used to evaluate previous word embedding methods~\citep{zhang2019learning, NIPS2014_feab05aa}. We introduce our training corpora and baselines in Section~\ref{data-section}. We describe the details of implementation and evaluation tasks in Section~\ref{imp.-section} and Section~\ref{eva-section}. We report experimental results and ablation study in Section~\ref{results-section} and Section~\ref{ablation-section}, respectively.

\begin{table}[htbp]
\caption{Data statistics of enwik9, WebBase and CC.
`\#tokens' indicates the number of total tokens in the whole corpora.
`\#words' indicates the number of unique words.
`\#pairs' indicates the number of word-context pairs.
}
\centering
\setlength{\tabcolsep}{1.8mm}{
\begin{tabular}{l|r|r|r}
\toprule
        & enwik9                & WebBase                & CC                \\ \midrule
\#tokens & 124,301,826            & 2,976,897,565          & 6,065,531,635          \\ 
\#words & 833,184                & 3,107,950              & 10,558,748             \\ 
\#pairs & 42,234,884             & 930,896,198            & 1,861,679,208          \\ \bottomrule
\end{tabular}
}
\label{tab:corpora_stat}
\end{table}

\subsection{Training Corpora and Baselines}
\label{data-section}

\textbf{Training corpora.} We conduct our experiments on three large scale English corpora: \textit{enwik9} with about 0.1 billion tokens, \textit{WebBase}~\cite{han-etal-2013-umbc} with about 3 billion tokens, and the filtering of 
\textit{Common Crawl corpus (CC})~\cite{ortiz-suarez-etal-2020-monolingual} with about 6 billion tokens. We first lowercase text and then remove the noisy text like HTML span.
We mainly implement data pre-processing on the basis of the script\footnote{\url{http://mattmahoney.net/dc/textdata.html}} used in word2vec.
After pre-processed, the vocabulary sizes of three corpora are 56,466 (on enwik9), 277,704 (on WebBase) and 400,000 (on CC) respectively. We prepare co-occurrence matrices with a context window size of 10. The co-occurrence matrices are square symmetric matrices whose row sizes are equal to vocabulary sizes. 
Our training corpora (enwik9\footnote{\url{http://mattmahoney.net/dc/enwik9.zip}},WebBase\footnote{\url{http://ebiquity.umbc.edu/redirect/to/resource/id/351/UMBC-webbase-corpus}},CC\footnote{\url{https://oscar-public.huma-num.fr/shuff-dedup/en/}}) are publicly available.
Note that the original CC is super-scale corpora, which obtains over 418B tokens. To facilitate experiments, we take 6B tokens.
The statistics of the three raw corpora are listed in Table~\ref{tab:corpora_stat}.

\textbf{Baselines.} In our experiments, we mainly compare our method {$\ouralgo$} with  popular benchmarks listed below.
\begin{itemize}
    \item \textbf{MF}~\citep{NIPS2014_feab05aa} is a global matrix factorization method that uncovers the semantic information implied in the matrix, such as PPMI and SPPMI matrices.
    \item \textbf{word2vec}\footnote{\url{https://code.google.com/archive/p/word2vec/source/default/source}}~\citep{DBLP:conf/nips/MikolovSCCD13} is a two-layer neural network that is trained to reconstruct  linguistic contexts of words. We choose the SGNS architecture in our experiments.
    \item \textbf{GloVe}\footnote{\url{https://github.com/stanfordnlp/GloVe}\label{glove}}~\citep{DBLP:conf/emnlp/PenningtonSM14} is a typical co-occurrence count based neural method, which captures global information from word co-occurrence matrix in the training corpus.
    \item  \textbf{fasttext}\footnote{\url{https://github.com/facebookresearch/fastText}}~\citep{fasttext} is a character-based method which represents each word as an n-gram of characters.
\end{itemize}

\subsection{Implementation Details}
\label{imp.-section}
In our experiments, we use code\footnote{\url{https://github.com/stanfordnlp/GloVe/blob/master/src/cooccur.c}} from GloVe to obtain co-occurrence matrix.
For 
MF and {\ouralgo}, matrices are stored in a list of lists (LIL) sparse matrix format to allow efficient row operations.
The shift parameter of SPPMI is set to $\negs$=5 on all training corpora.
Sample size $k$ for our method $\ouralgo$ are chosen from [10,~100000] as shown in Fig.~\ref{fig:diff_sample_all}. We factorize the Q-contexts matrix $\tilde C$\ to get singular vectors and singular values.
According to the official guide, we use 15 iterations to train GloVe and word2vec methods and 5 iterations to train fasttext method.
The number of dimensions of embedding is set to 300.
Note that we set the dimension of word embedding to sample size $k$ when $k$ is smaller than 300 in ablation study of sample size.
All experiments are carried out on a cloud server with Intel(R) Xeon(R) Gold 6231C CPU. We set the number of CPU cores to 20, 20, 10, 1 and 1 for fasttext, word2vec, GloVe, MF and $\ouralgo$, respectively.

\begin{table*}[ht]
\centering
\caption{The performance of {\ouralgo} and all baselines on word similarity, text classification and NER tasks.
$k$ = the sample size in \ouralgo. By the 'random' method, we  randomly initialize the embedding of each word. $\text{GloVe}^\dagger$ is publicly released word embedding pre-trained with 6B tokens. 
\textbf{Bold scores} are best within groups of baselines and {\ouralgo}.
}
\vskip 0.15in
\setlength{\tabcolsep}{1.95mm}{
\begin{tabular}{c|c|r|ll|ll|ll} 
\toprule
\multirow{2}{*}{Dataset} & \multirow{2}{*}{Method} & \multicolumn{1}{c|}{\multirow{2}{*}{Time~/~mins ${\downarrow}$}} & \multicolumn{2}{c|}{Word similarity ${\uparrow}$}                                            & \multicolumn{2}{c|}{Text classification ${\uparrow}$}           & \multicolumn{2}{c}{NER $\uparrow$}     \\ 
\cmidrule{4-9}
                         &                        & \multicolumn{1}{c|}{(Cc: CPU core)}                      & \multicolumn{1}{c}{MEN} &\multicolumn{1}{c|}{WS353}  &  \multicolumn{1}{c}{SST-2} & \multicolumn{1}{c|}{5AG} & \multicolumn{1}{c}{Conll}  & \multicolumn{1}{c}{BTC} \\ 
\midrule
\multirow{6}{*}{\shortstack{enwik9\\(0.1B)}} 
     &MF-PPMI              & 27.4 $\times$ 1 Cc                                            & \textbf{74.24} & \textbf{70.31}                      & 78.61 & 86.63                          & 87.72 &  71.06 \\
     & $\ouralgo$-PPMI ($k$=40000)       & 14.9  $\times$ 1 Cc                                           & 73.19$\pm$0.28 & 69.23$\pm$0.41                      & 78.73$\pm$0.57 & \textbf{87.25}$\pm$0.29                          & \textbf{87.74}$\pm$0.11 & \textbf{71.16}$\pm$0.49 \\
     &MF-SPPMI              & 13.2 $\times$ 1 Cc                                            & 72.26 & 67.14                      & 78.54 & 86.19                          & 86.13 &  69.88  \\
     &$\ouralgo$-SPPMI ($k$=40000)       & \textbf{3.9  $\times$ 1 Cc}                                            & 71.92$\pm$0.58 & 66.52$\pm$1.60                      & 78.15$\pm$0.20 & 87.18$\pm$0.30                          & 86.34$\pm$0.12 &  69.70$\pm$0.12 \\ 
& GloVe                  & 44.6 $\times$ 10 Cc                                  & 69.45 & 66.31                      & 79.26 & 86.51                          & 86.78 & 71.00  \\
                         & word2vec                   & 59.9  $\times$ 20 Cc                                & 74.09 & 68.96                      & 79.13 & 85.62                          & 86.51 &  70.19  \\
                         & fasttext                   & 67.3 $\times$ 20 Cc                                & 73.13 & 65.79                      & \textbf{80.59} & 86.85                          & 87.63 &  70.65  \\

\midrule
\multirow{6}{*}{\shortstack{WebBase\\(3B)}} 
 & MF-PPMI               & 185.4 $\times$ 1 Cc                                           & \textbf{77.97}         & \textbf{68.77}                           & 80.94                   & 86.32                   & 88.21            &        71.31   \\
 & $\ouralgo$-PPMI ($k$=50000)                 & 28.4 $\times$ 1 Cc                                            & 77.01$\pm$0.44                  & 68.74$\pm$0.84                        & 81.59$\pm$0.14                   & 87.17$\pm$0.22                   & 88.48$\pm$0.17         &       71.37$\pm$0.49       \\ 
 & MF-SPPMI               & 67.6 $\times$ 1 Cc                   & 76.84                  & 67.62                                        & 79.13                   & 85.34                   & 87.07            &     69.31      \\
 & $\ouralgo$-SPPMI ($k$=50000)                 & \textbf{7.3 $\times$ 1 Cc}                              & 76.31$\pm$0.55                  & 67.66$\pm$0.55                                  & 78.98$\pm$0.28                   & 86.86$\pm$0.23                   & 87.68$\pm$0.12         &       70.44$\pm$0.24       \\
& GloVe                  & 182.7×10 Cc                                      & 75.20 & 63.43                      & \textbf{82.66} & 86.90                           & 87.78 &  71.02  \\
                         & word2vec                   &     1211.4 $\times$ 20 Cc                                             &  74.73      & 63.98                           & 80.93 & 87.64                          & 88.58 & 69.89  \\
& fasttext                   & 1843.7  $\times$ 20 Cc                                & 73.10 & 58.80                      & 81.07 & \textbf{87.74}                          & \textbf{88.93} &  \textbf{72.83}  \\
\midrule
\multirow{6}{*}{\shortstack{CC\\(6B)}}     
 &MF-PPMI              & 406.5 $\times$ 1 Cc                                           & 77.72 & \textbf{70.87}                      & \textbf{83.99} & 86.85                          & 87.48  &  72.55  \\
 &$\ouralgo$-PPMI ($k$=50000)       & 31.4 $\times$ 1 Cc                                             & 76.39$\pm$0.54 & 69.54$\pm$0.54                      & 82.90$\pm$0.74 & \textbf{87.69}$\pm$0.15                          & 88.59$\pm$0.25 & 72.36$\pm$0.36  \\
 &MF-SPPMI              & 153.9  $\times$ 1 Cc                                          & 75.30 & 69.35                      & 80.25 & 85.92                          & 87.95 & 70.37  \\
 &$\ouralgo$-SPPMI ($k$=50000)       & \textbf{13.3 $\times$ 1 Cc}                                          & 73.92$\pm$0.35 & 66.15$\pm$0.96                      & 80.79$\pm$0.66 & 86.36$\pm$0.17                          & 87.78$\pm$0.18  & 70.86$\pm$0.56  \\
& GloVe                  & 270.5 $\times$ 10 Cc                                             & 76.37      &  65.41                           &  81.70      & 86.47                               & 88.56  &   71.23  \\
& word2vec                   & 2369.7  $\times$ 20 Cc                               & 79.29 & 69.90                      & 81.57 & 87.14                          & 88.65  & 72.93 \\
& fasttext                   & 3419.4  $\times$ 20 Cc                                & \textbf{79.64} & 68.97                      & 82.70 & 87.53                          & \textbf{89.46} &  \textbf{73.92}  \\
\midrule
\multirow{2}{*}{}      & random                  &             \multicolumn{1}{c|}{/}                                & 0.31      &  2.38                            & 73.82                   & 82.51                   & 81.63      & 59.22 \\
                         & $\text{GloVe}^\dagger$                   &              \multicolumn{1}{c|}{/}                  & 73.75 & 57.26                      & 82.79 & 87.34                          & 91.15   &  72.26  \\

\bottomrule
\end{tabular}
 \label{tab:results}}
\end{table*}

\subsection{Evaluation Tasks}
\label{eva-section}
To measure the quality of embeddings, we conduct experiments on three different evaluation tasks: Word Similarity, Text Classification and Named Entity Recognition~(NER).

\textbf{Word similarity.}
The word similarity task is to measure how well the semantic relationship between words is captured by word embeddings.
Experiments are conducted on 
MEN~\citep{bruni2014multimodal} and WS353~\citep{agirre-etal-2009-study} datasets. 
To estimate the quality of word embeddings, we compute the Spearman Rank Correlation score between the word embedding cosine similarities and human-annotated scores.

\textbf{Text classification.}
In the text classification task, the classifier predicts predefined labels for the given texts.
We conduct experiments on two datasets, i.e., 5AbstractsGroup (5AG)~\cite{liu-etal-2018-task} and Stanford Sentiment Treebank (SST-2)~\cite{SST2}.We choose TextCNN~\cite{kim-2014-convolutional} as our classifier, and measure the performance by weighted F1 metrics. 
The sizes of the kernels are 2, 3 and 4 respectively. 
We use the Adam optimizer with a mini-batch size of 128 and learning rate of 0.001.

\textbf{Named entity recognition.}
NER is the task of identifying and categorizing the entities in the given text. We conduct experiments on two benchmark datasets: Conll~\cite{tjong-kim-sang-de-meulder-2003-introduction} and BTC~\cite{derczynski2016broad}. 
We choose wordLSTM+charCNN+CRF~\cite{ma-hovy-2016-end} as our NER model and measure the performance by entity-level F1 metrics.
The wordLSTM is built upon a one-layer bidirectional LSTM structure with 200 hidden size. We employ the SGD algorithm to optimize model parameters on mini-batches of size 100, with a learning rate of 0.001.

\subsection{Results}
\label{results-section}

Table~\ref{tab:results} illustrates the results on various evaluation tasks and training times of all word embedding methods. It can be observed that our method $\ouralgo$ outperforms all baselines in three corpora on training time (and resources) metric while the performance metrics are fairly close.

In terms of evaluation tasks, we notice that MF and $\ouralgo$ are superior to GloVe, word2vec and fasttext in most word similarity tasks, whereas performing  slightly worse on text classification and NER tasks. 
For both NLP tasks (Text classification and NER), models using pre-trained word embeddings have remarkable improvement compared to the one using random embeddings. 
Comparing with MF, results obtained with $\ouralgo$ are fairly close to the ones obtained with MF and sometimes even better, despite the fact we only use a fraction of the original information matrix. 

We also observe that the released GloVe embedding is about 3\% better than word embeddings trained by us on Conll task. 
One possible reason is the mismatch of vocabulary~\cite{ma-hovy-2016-end}. We reuse the data pre-process script of word2vec excluding punctuations and digits. And the punctuations and digits are important in the Conll task.

To verify the stability of our method {\ouralgo}, we run 5 times for each embedding. We only compute the mean and variance of {\ouralgo} because other methods need a long time to train word embedding. The small variance shown in Table \ref{tab:results} indicate that the word embedding obtained by {\ouralgo} are stable.

With regard to training time, MF method innately has an advantage over neural network based methods (word2vec, GloVe and fasttext) given the same computational settings. Our method $\ouralgo$ improves significantly on top of that. As can be seen from the Table~\ref{tab:results}, our method $\ouralgo$-PPMI only requires 54.38\% (on enwik9), 15.32\% (on WebBase) and 7.72\% (on CC) of the training time of MF-PPMI respectively. With the CC corpus, while GloVe needs 270.5 minutes with 10 CPU cores to obtain word embeddings, it takes only 31.4 minutes with 1 CPU core with $\ouralgo$-PPMI. Although, as the vocabulary size increases, the training time of $\ouralgo$ increases accordingly, the time taken by MF increases by a significantly larger margin. 

Note that in the case closest to practice (6B tokens and 400K vocabularies), we achieve at least $11\sim 13$ times speed-up compared with other baselines. 
In more realistic scenarios where the number of tokens is trillions and the size of vocabulary is several million, one can infer from the growing trend that \emph{the speed advantage of our algorithm could potentially reach several orders of magnitude.}

The analysis presented above confirms that 
$\ouralgo$~method can substantially minimize the training time while maintaining benchmark performance. 

\begin{figure*}[htbp]
    \centering
    \centerline{\includegraphics[scale=0.50]{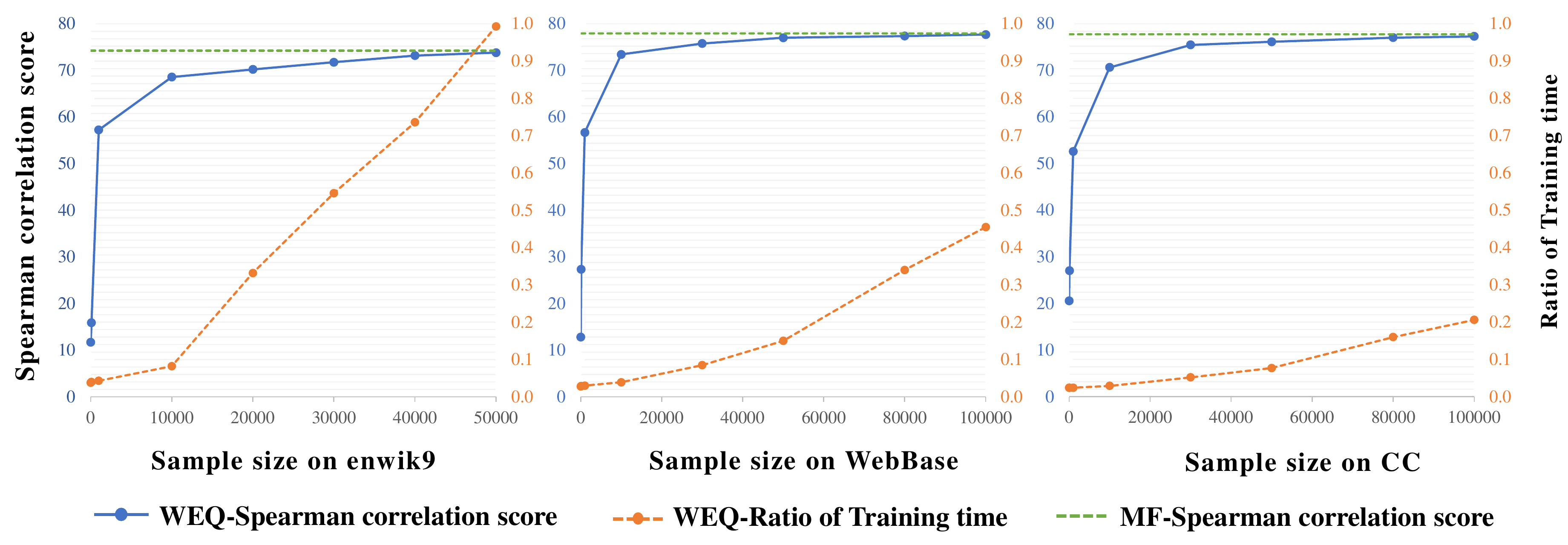}}
    \caption{The scores of similar task (MEN) and the corresponding training times of {\ouralgo}-PPMI method under different sample sizes. The original dimensions of the information  matrices are about 56K (on enwik9), 278K (on WebBase) and 400K (on CC) respectively. Ratio of Training time = training time of \ouralgo~/ training time of MF.}
    \label{fig:diff_sample_all}
\end{figure*}

\subsection{Discussion and Analysis}
\label{ablation-section}

\begin{table}[htbp]
\caption{The training time decomposition of the different part of {$\ouralgo$} methods on the CC corpus. Total training time of {$\ouralgo$}-PPMI($k$=50000) and {$\ouralgo$}-SPPMI($k$=50000) are 31.4 minutes and 13.3 minutes respectively.}
\resizebox{0.48\textwidth}{!}{
\scriptsize
\begin{tabular}{@{}l@{}ccccc@{}}
\toprule
          & \multicolumn{1}{l@{}}{\begin{tabular}[c]{@{}l@{}}Info. matrix\\ computation\end{tabular}} & \multicolumn{1}{l@{}}{\begin{tabular}[c]{@{}l@{}}State\\ preparation\end{tabular}} & \multicolumn{1}{l@{}}{Sampling} & \multicolumn{1}{l@{}}{\begin{tabular}[c]{@{}l@{}}Sing. val.\\ vec.\end{tabular}} & \multicolumn{1}{l@{}}{\begin{tabular}[c]{@{}l@{}}Embedding\\ calculation\end{tabular}}   \\
\midrule
\ouralgo-PPMI  & 26.37\%                                                                                & 3.74\%                                                                          & 6.93\%                       & 46.40\%                                                                        & 16.56\%                                                                             \\
\ouralgo-SPPMI & 52.93\%                                                                                & 3.76\%                                                                          & 3.73\%                       & 30.92\%                                                                        & 8.66\%                                                                              \\
\bottomrule
\end{tabular}
}
\label{tab:time_dis}
\end{table}

\textbf{Time analysis.} 
Recall three main steps of the \ouralgo~method:
information matrix computation, Q-contexts construction (including state preparation and sampling) and calculating word embedding from Q-contexts (including computation of singular values and left singular vectors, and embedding computation). The breakdown of computational time is displayed in Table.\ref{tab:time_dis}. 
Note that the Q-contexts construction step takes up only no more than 11\% of the whole time, 
which shows the efficiency of Q-contexts construction step.

\textbf{Sample size.} In our method, Q-contexts matrix is obtained by sampling from the information matrix.  
To investigate how the sample size affects the performance of $\ouralgo$, we carry out experiments with different sample sizes on all three corpora as shown in Fig.\ref{fig:diff_sample_all}. We present a representative benchmark score, similarity task MEN, against the training times under corresponding sample sizes. The results of other evaluation tasks are shown in Fig.\ref{fig:sample_all_evaluation}.
As the sample size increases, it is obvious that the performance of $\ouralgo$ on MEN approaches that of MF. The convergence is fairly fast, especially when the corpus size is large. We can see that to reach at least 98\% of the performance of MF, the sample size of $\ouralgo$ method only needs to be 40k (on enwik9), 50k (on WebBase), 50k (on CC), which are about 70.83\%, 18.00\%, 12.50\% of the original 
size of the information matrices, respectively. This validates the theoretical observations in Section~\ref{error-analysis}.
As the size of the original information matrix increases, the percentage of samples needed decreases. In terms of training time, it rises slowly at small sample sizes then trends upward sharply at larger sample sizes. This indicates that the training time can be effectively reduced when sampling a small matrix to get word embedding.

\begin{figure}[htbp]
    \centering
    \centerline{\includegraphics[scale=0.36]{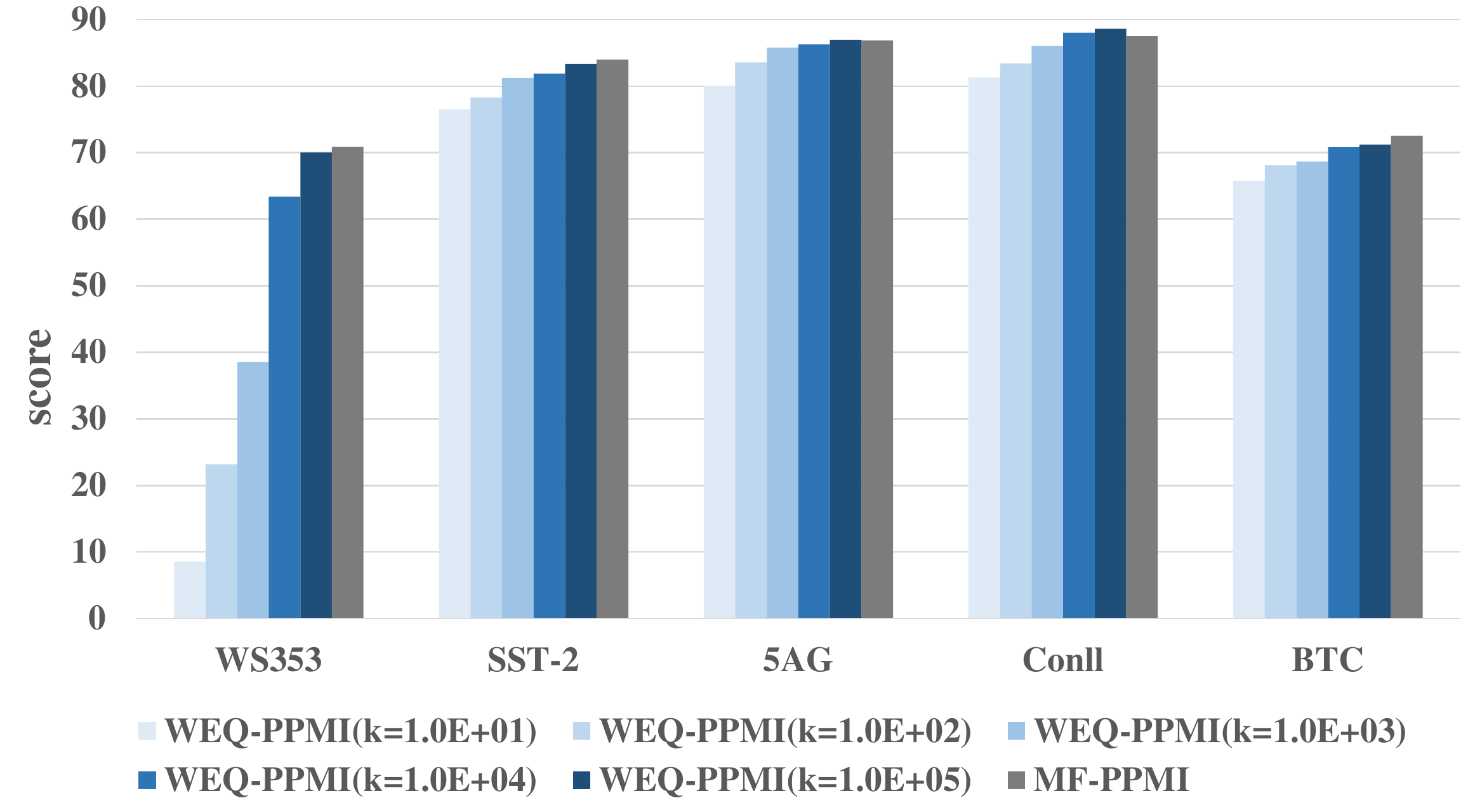}}
    \caption{The scores of all evaluation tasks except MEN under different sample sizes on CC corpus.}
    \label{fig:sample_all_evaluation}
\end{figure}

\textbf{$\ell^2$-norm sampling vs uniform sampling.} In $\ouralgo$ method, we use $\ell^2$-norm sampling to capture the typical information of origin information matrix. To verify the effectiveness of the $\ell^2$-norm sampling, we replace the $\ell^2$-norm sampling with uniform sampling in the Q-contexts construction step. Uniform sampling means that each row of information matrix $M$ will be sampled with equal probability. We compare the performance of word embeddings trained by these two different sampling algorithms. The evaluation results are shown in Fig.~\ref{fig:uniform}. $\ell^2$-norm sampling has great advantages over uniform sampling in terms of word similarity scores. In terms of downstream tasks (text classification, NER), $\ell^2$-norm sampling has $1.5\% \sim 4\%$ increase compared with uniform sampling. These experimental results are consistent with our theoretical proof in Section~\ref{proof-section}.

\begin{figure}[htbp]
    \begin{center}
    \includegraphics[scale=0.36]{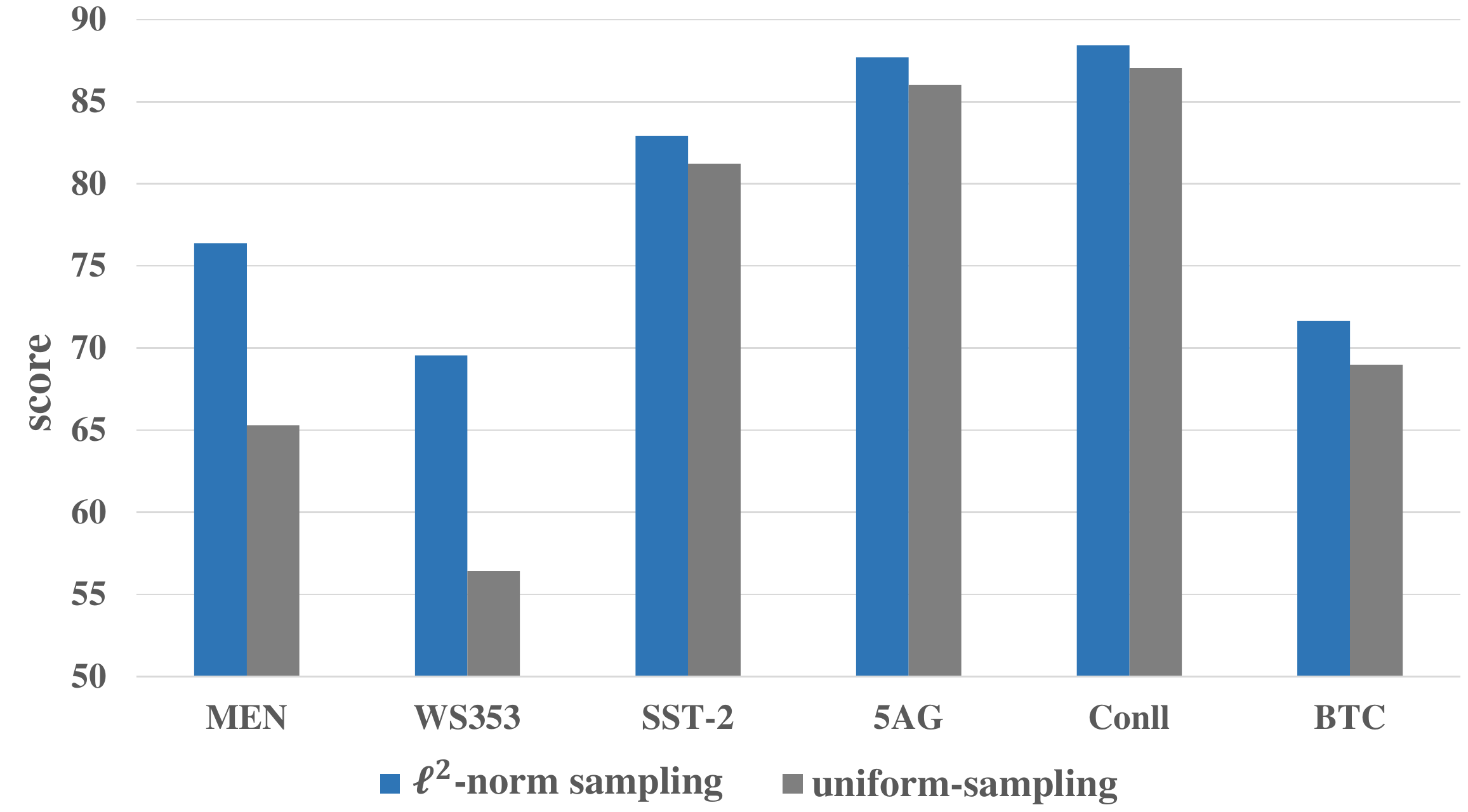}
    \caption{Comparison between $\ell^2$-norm sampling and uniform sampling.}
    \label{fig:uniform}
    \end{center}
\end{figure}

\textbf{Nonzeros/Sparsity analysis.}
To investigate the correlation between matrix sparsity and algorithm efficiency, we show a comparison between time and 
number of nonzeros in an information matrix $M$ in Fig.~\ref{fig:nnz}.
By $\text{TES}$, we mean total time excluding SVD computation, which corresponds to Algorithm~3 excluding line 3.
Since the time for SVD computation depends on the SVD engine used in our method, we do not provide a detailed analysis.
The Pearson Correlation Score between $\text{nnz}(M)$ and $\text{TES}$ is 0.9583.
It demonstrates that our algorithm efficiency is approximately linearly and positively correlated with the number of nonzero elements of $M$.
We also observe that number of nonzero elements in SPPMI matrix is significantly less than that in PPMI. That is because only the positive values higher than $\kappa$ are retained after the shifted operation. The decrease of number of nonzero elements leads to a more sparse matrix and shorter $\text{TES}$.

\begin{figure}[htbp]
    \begin{center}
    \includegraphics[scale=0.34]{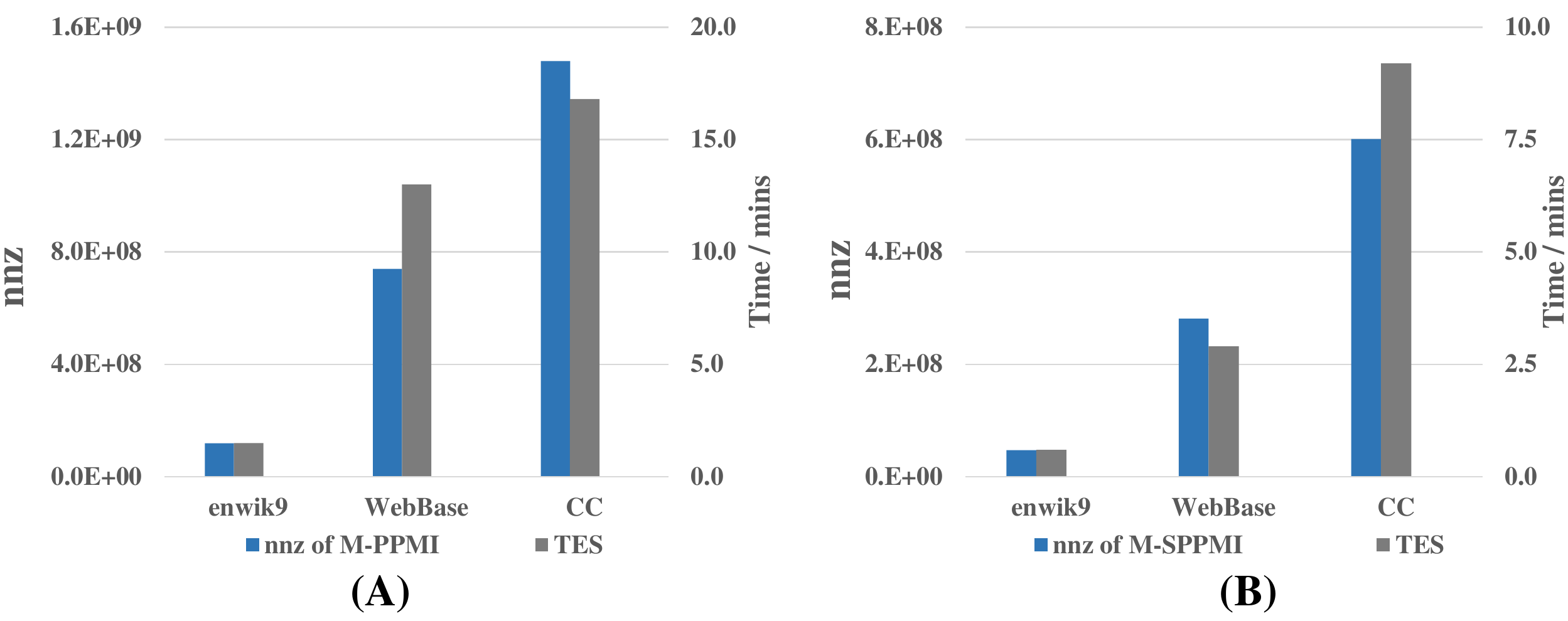}
    \caption{Comparison between $\text{nnz}(M)$ and $\text{TES}$ on three different training corpora. $\text{nnz}(M)$ means number of nonzero elements in information matrix $M$. $\text{TES}$ means  algorithm training Time Excluding SVD.
    The Pearson Correlation Score is 0.9583.
    Panel A is the result for the PPMI information matrix, and Panel B is the  result for the SPPMI matrix. }
    \label{fig:nnz}
    \end{center}
\end{figure}

\section{Related work}
\label{related word-section}

In this section, we review the related work of word embedding and low-rank approximation.

\subsection{Word Embedding}
Constructing the word embedding which could express the semantic features is a fundamental task in Natural Language Processing. It brings benefits to many NLP tasks~\citep{maruf-haffari-2018-document, khabiri2019industry}. 
Briefly, recent work about word embedding can be categorized into two genres, i.e., neural network based methods~\cite{DBLP:conf/nips/MikolovSCCD13, fasttext,DBLP:conf/emnlp/PenningtonSM14} and global matrix factorization based methods~\cite{deerwester1990indexing,NIPS2014_feab05aa,arora}. It is found that we can bridge these two categories by unifying neural network based methods into a matrix factorization framework~\citep{arora, NIPS2014_feab05aa}.

 Most of the neural network based methods model the relation between the target word and its contextual words. Among them, the SGNS (Skip-Gram with Negative Sampling) model is the popular implementation of word2vec~\cite{DBLP:conf/nips/MikolovSCCD13}. It separates local context windows on the whole corpus, and focuses on maximizing the likelihood of contextual words based on the given word. Fasttext~\cite{fasttext} is an extension of the word2vec model, which represents each word as an n-gram of characters by a sliding window. GloVe~\cite{DBLP:conf/emnlp/PenningtonSM14} is trained based on the global word-word co-occurrence counts from a corpus.
As for one kind of global matrix factorization based method, MF~\cite{NIPS2014_feab05aa} factorizes the SPPMI matrix explicitly, rather than iteratively tuning the network parameters.

More recently, ELMo~\cite{Elmo} and BERT~\cite{devlin2019bert} achieve state-of-the-art results in a variety of NLP tasks.
Despite the state-of-the-art results achieved by deep contextualization models, pre-trained word embedding should not be neglected. Indeed, ELMo is dependent on GloVe embedding as input during training~\cite{Elmo}. Moreover, both ELMo and BERT are extremely large deep networks that require huge computing resources. The pre-training of $\text{BERT}_{large}$ takes 4 days to complete on 64 TPUs~\citep{devlin2019bert}. And it is also difficult
to apply deep contextualization models directly on low-resource PCs or mobile phones~\cite{strubell-etal-2019-energy, sun-etal-2020-mobilebert}.

In the studies described above, researchers paid less attention to the computational cost of word embedding training. It is true that large-scale word embedding training requires significant computational costs~\cite{DBLP:conf/nips/MikolovSCCD13, Mnih, DBLP:conf/emnlp/PenningtonSM14, mikolov13}.
In this work, we aim to present a novel efficient method to obtain word embeddings of a large-scale vocabulary, while maintaining its superiority in terms of effectiveness.

\subsection{Low-Rank Approximation}
Low-rank approximation problems in mathematical modeling have been studied for decades~\citep{fkv,woodruff}, which is approximating a matrix by one whose rank is less than that of the original matrix. The aim is to obtain more compact information representation and less complex data modeling with limited losses. It arises in many applications such as recommendation system~\citep{xanadu, Ewin}. 
Unlike their implementations of random matrices and matrices of small sizes~\citep{xanadu}, our \ouralgo~investigates the actual large matrices associated with natural language and reveals new relations for word embeddings.
Our work is also different from \cite{Ewin, fkv} since we provide more detailed theoretical and empirical analysis on matrix concentration, and direct analysis of singular vectors.
We also have novel treatments for sparse matrices and better algorithmic designs for state preparation.

\section{Conclusion}
\label{sec:conclusion}
In this work, we introduce the notion of Q-contexts (matrix), which can be constructed efficiently.
These are only a small fraction (less than 12.5\% in the practical scenario) of all the contexts in the entire corpus.
We also present a novel relation between word vectors and Q-contexts, and provide a theoretical foundation and rigorous analysis.
Based on this relation, we design a novel and efficient \ouralgo~method for fast computation of large-scale word embedding. 
Empirical experiments show that our algorithm \ouralgo~method runs at least 
$11\sim 13$ times
faster than well-established matrix factorization methods.
Resource and time advantages over word2vec, GloVe and fasttext are even more pronounced in our empirical study. We have also shown that \ouralgo~enjoys decent accuracy performance in a variety of NLP tasks compared to the other methods tested in this study.
In the future, we would like to efficiently learn word embedding of million-level vocabulary. 
It is also an interesting topic to apply our method to other domains such as graph embedding.

\begin{acks}
This work was supported by National Natural Science Foundation of China (62076100), and Fundamental Research Funds for the Central Universities, SCUT (D2210010,D2200150,and D2201300), the Science and Technology Planning Project of Guangdong Province (2020B0101100002).
\end{acks}

\bibliographystyle{ACM-Reference-Format}
\balance
\bibliography{cikm}


\begin{thebibliography}{39}


\ifx \showCODEN    \undefined \def \showCODEN     #1{\unskip}     \fi
\ifx \showDOI      \undefined \def \showDOI       #1{#1}\fi
\ifx \showISBNx    \undefined \def \showISBNx     #1{\unskip}     \fi
\ifx \showISBNxiii \undefined \def \showISBNxiii  #1{\unskip}     \fi
\ifx \showISSN     \undefined \def \showISSN      #1{\unskip}     \fi
\ifx \showLCCN     \undefined \def \showLCCN      #1{\unskip}     \fi
\ifx \shownote     \undefined \def \shownote      #1{#1}          \fi
\ifx \showarticletitle \undefined \def \showarticletitle #1{#1}   \fi
\ifx \showURL      \undefined \def \showURL       {\relax}        \fi
\providecommand\bibfield[2]{#2}
\providecommand\bibinfo[2]{#2}
\providecommand\natexlab[1]{#1}
\providecommand\showeprint[2][]{arXiv:#2}

\bibitem[\protect\citeauthoryear{Agirre, Alfonseca, Hall, Kravalova,
  Pa{\c{s}}ca, and Soroa}{Agirre et~al\mbox{.}}{2009}]%
        {agirre-etal-2009-study}
\bibfield{author}{\bibinfo{person}{Eneko Agirre}, \bibinfo{person}{Enrique
  Alfonseca}, \bibinfo{person}{Keith Hall}, \bibinfo{person}{Jana Kravalova},
  \bibinfo{person}{Marius Pa{\c{s}}ca}, {and} \bibinfo{person}{Aitor Soroa}.}
  \bibinfo{year}{2009}\natexlab{}.
\newblock \showarticletitle{A Study on Similarity and Relatedness Using
  Distributional and {W}ord{N}et-based Approaches}. In
  \bibinfo{booktitle}{\emph{NAACL '09}}.
\newblock


\bibitem[\protect\citeauthoryear{Arora, Li, Liang, Ma, and Risteski}{Arora
  et~al\mbox{.}}{2016}]%
        {arora}
\bibfield{author}{\bibinfo{person}{Sanjeev Arora}, \bibinfo{person}{Yuanzhi
  Li}, \bibinfo{person}{Yingyu Liang}, \bibinfo{person}{Tengyu Ma}, {and}
  \bibinfo{person}{Andrej Risteski}.} \bibinfo{year}{2016}\natexlab{}.
\newblock \showarticletitle{A latent variable model approach to pmi-based word
  embeddings}.
\newblock \bibinfo{journal}{\emph{Transactions of the Association for
  Computational Linguistics}}  \bibinfo{volume}{4} (\bibinfo{year}{2016}),
  \bibinfo{pages}{385--399}.
\newblock


\bibitem[\protect\citeauthoryear{Arrazola, Delgado, Bardhan, and
  Lloyd}{Arrazola et~al\mbox{.}}{2020}]%
        {xanadu}
\bibfield{author}{\bibinfo{person}{Juan~Miguel Arrazola},
  \bibinfo{person}{Alain Delgado}, \bibinfo{person}{Bhaskar~Roy Bardhan}, {and}
  \bibinfo{person}{Seth Lloyd}.} \bibinfo{year}{2020}\natexlab{}.
\newblock \showarticletitle{Quantum-inspired algorithms in practice}.
\newblock \bibinfo{journal}{\emph{Quantum}} (\bibinfo{year}{2020}).
\newblock


\bibitem[\protect\citeauthoryear{Bojanowski, Grave, Joulin, and
  Mikolov}{Bojanowski et~al\mbox{.}}{2017}]%
        {fasttext}
\bibfield{author}{\bibinfo{person}{Piotr Bojanowski}, \bibinfo{person}{Edouard
  Grave}, \bibinfo{person}{Armand Joulin}, {and} \bibinfo{person}{Tomas
  Mikolov}.} \bibinfo{year}{2017}\natexlab{}.
\newblock \showarticletitle{Enriching word vectors with subword information}.
\newblock \bibinfo{journal}{\emph{Transactions of the Association for
  Computational Linguistics}}  \bibinfo{volume}{5} (\bibinfo{year}{2017}),
  \bibinfo{pages}{135--146}.
\newblock


\bibitem[\protect\citeauthoryear{Bruni, Tran, and Baroni}{Bruni
  et~al\mbox{.}}{2014}]%
        {bruni2014multimodal}
\bibfield{author}{\bibinfo{person}{Elia Bruni}, \bibinfo{person}{Nam-Khanh
  Tran}, {and} \bibinfo{person}{Marco Baroni}.}
  \bibinfo{year}{2014}\natexlab{}.
\newblock \showarticletitle{Multimodal distributional semantics}.
\newblock \bibinfo{journal}{\emph{Journal of artificial intelligence research}}
   \bibinfo{volume}{49} (\bibinfo{year}{2014}), \bibinfo{pages}{1--47}.
\newblock


\bibitem[\protect\citeauthoryear{Church and Hanks}{Church and Hanks}{1989}]%
        {church}
\bibfield{author}{\bibinfo{person}{Kenneth~Ward Church} {and}
  \bibinfo{person}{Patrick Hanks}.} \bibinfo{year}{1989}\natexlab{}.
\newblock \showarticletitle{Word Association Norms, Mutual Information, and
  Lexicography}. In \bibinfo{booktitle}{\emph{ACL '89}}.
\newblock


\bibitem[\protect\citeauthoryear{Coecke, de~Felice, Meichanetzidis, and
  Toumi}{Coecke et~al\mbox{.}}{2020}]%
        {coecke2020foundations}
\bibfield{author}{\bibinfo{person}{Bob Coecke}, \bibinfo{person}{Giovanni de
  Felice}, \bibinfo{person}{Konstantinos Meichanetzidis}, {and}
  \bibinfo{person}{Alexis Toumi}.} \bibinfo{year}{2020}\natexlab{}.
\newblock \showarticletitle{Foundations for Near-Term Quantum Natural Language
  Processing}.
\newblock \bibinfo{journal}{\emph{arXiv preprint arXiv:2012.03755}}
  (\bibinfo{year}{2020}).
\newblock


\bibitem[\protect\citeauthoryear{Dagan, Pereira, and Lee}{Dagan
  et~al\mbox{.}}{1994}]%
        {dagan}
\bibfield{author}{\bibinfo{person}{Ido Dagan}, \bibinfo{person}{Fernando
  Pereira}, {and} \bibinfo{person}{Lillian Lee}.}
  \bibinfo{year}{1994}\natexlab{}.
\newblock \showarticletitle{Similarity-Based Estimation of Word Cooccurrence
  Probabilities}. In \bibinfo{booktitle}{\emph{ACL '94}}.
\newblock


\bibitem[\protect\citeauthoryear{Dahiya, Konomis, and Woodruff}{Dahiya
  et~al\mbox{.}}{2018}]%
        {woodruff}
\bibfield{author}{\bibinfo{person}{Yogesh Dahiya}, \bibinfo{person}{Dimitris
  Konomis}, {and} \bibinfo{person}{David~P. Woodruff}.}
  \bibinfo{year}{2018}\natexlab{}.
\newblock \showarticletitle{An Empirical Evaluation of Sketching for Numerical
  Linear Algebra}. In \bibinfo{booktitle}{\emph{KDD '18}}.
\newblock


\bibitem[\protect\citeauthoryear{Deerwester, Dumais, Furnas, Landauer, and
  Harshman}{Deerwester et~al\mbox{.}}{1990}]%
        {deerwester1990indexing}
\bibfield{author}{\bibinfo{person}{Scott Deerwester}, \bibinfo{person}{Susan~T
  Dumais}, \bibinfo{person}{George~W Furnas}, \bibinfo{person}{Thomas~K
  Landauer}, {and} \bibinfo{person}{Richard Harshman}.}
  \bibinfo{year}{1990}\natexlab{}.
\newblock \showarticletitle{Indexing by latent semantic analysis}.
\newblock \bibinfo{journal}{\emph{Journal of the American society for
  information science}} \bibinfo{volume}{41}, \bibinfo{number}{6}
  (\bibinfo{year}{1990}), \bibinfo{pages}{391--407}.
\newblock


\bibitem[\protect\citeauthoryear{Derczynski, Bontcheva, and Roberts}{Derczynski
  et~al\mbox{.}}{2016}]%
        {derczynski2016broad}
\bibfield{author}{\bibinfo{person}{Leon Derczynski}, \bibinfo{person}{Kalina
  Bontcheva}, {and} \bibinfo{person}{Ian Roberts}.}
  \bibinfo{year}{2016}\natexlab{}.
\newblock \showarticletitle{Broad twitter corpus: A diverse named entity
  recognition resource}. In \bibinfo{booktitle}{\emph{Proceedings of COLING
  2016, the 26th International Conference on Computational Linguistics:
  Technical Papers}}. \bibinfo{pages}{1169--1179}.
\newblock


\bibitem[\protect\citeauthoryear{Devlin, Chang, Lee, and Toutanova}{Devlin
  et~al\mbox{.}}{2019}]%
        {devlin2019bert}
\bibfield{author}{\bibinfo{person}{Jacob Devlin}, \bibinfo{person}{Ming-Wei
  Chang}, \bibinfo{person}{Kenton Lee}, {and} \bibinfo{person}{Kristina
  Toutanova}.} \bibinfo{year}{2019}\natexlab{}.
\newblock \showarticletitle{BERT: Pre-training of Deep Bidirectional
  Transformers for Language Understanding}. In \bibinfo{booktitle}{\emph{NAACL
  '19}}. \bibinfo{pages}{4171--4186}.
\newblock


\bibitem[\protect\citeauthoryear{Drineas, Mahoney, and Muthukrishnan}{Drineas
  et~al\mbox{.}}{2008}]%
        {DBLP:journals/siammax/DrineasMM08}
\bibfield{author}{\bibinfo{person}{Petros Drineas}, \bibinfo{person}{Michael~W.
  Mahoney}, {and} \bibinfo{person}{S. Muthukrishnan}.}
  \bibinfo{year}{2008}\natexlab{}.
\newblock \showarticletitle{Relative-Error {CUR} Matrix Decompositions}.
\newblock \bibinfo{journal}{\emph{{SIAM} J. Matrix Anal. Appl.}}
  \bibinfo{volume}{30}, \bibinfo{number}{2} (\bibinfo{year}{2008}),
  \bibinfo{pages}{844--881}.
\newblock
\urldef\tempurl%
\url{https://doi.org/10.1137/07070471X}
\showDOI{\tempurl}


\bibitem[\protect\citeauthoryear{Frieze, Kannan, and Vempala}{Frieze
  et~al\mbox{.}}{2004}]%
        {fkv}
\bibfield{author}{\bibinfo{person}{Alan Frieze}, \bibinfo{person}{Ravi Kannan},
  {and} \bibinfo{person}{Santosh Vempala}.} \bibinfo{year}{2004}\natexlab{}.
\newblock \showarticletitle{Fast Monte-Carlo Algorithms for Finding Low-Rank
  Approximations}.
\newblock \bibinfo{journal}{\emph{J. ACM}} (\bibinfo{year}{2004}).
\newblock


\bibitem[\protect\citeauthoryear{Han, L.~Kashyap, Finin, Mayfield, and
  Weese}{Han et~al\mbox{.}}{2013}]%
        {han-etal-2013-umbc}
\bibfield{author}{\bibinfo{person}{Lushan Han}, \bibinfo{person}{Abhay
  L.~Kashyap}, \bibinfo{person}{Tim Finin}, \bibinfo{person}{James Mayfield},
  {and} \bibinfo{person}{Jonathan Weese}.} \bibinfo{year}{2013}\natexlab{}.
\newblock \showarticletitle{{UMBC}{\_}{EBIQUITY}-{CORE}: Semantic Textual
  Similarity Systems}. In \bibinfo{booktitle}{\emph{Second Joint Conference on
  Lexical and Computational Semantics (*{SEM}), Volume 1: Proceedings of the
  Main Conference and the Shared Task: Semantic Textual Similarity}}.
\newblock


\bibitem[\protect\citeauthoryear{Hazan, Koren, and Srebro}{Hazan
  et~al\mbox{.}}{2011}]%
        {DBLP:conf/nips/HazanKS11}
\bibfield{author}{\bibinfo{person}{Elad Hazan}, \bibinfo{person}{Tomer Koren},
  {and} \bibinfo{person}{Nati Srebro}.} \bibinfo{year}{2011}\natexlab{}.
\newblock \showarticletitle{Beating {SGD:} Learning SVMs in Sublinear Time}. In
  \bibinfo{booktitle}{\emph{Advances in Neural Information Processing Systems
  24: 25th Annual Conference on Neural Information Processing Systems 2011.
  Proceedings of a meeting held 12-14 December 2011, Granada, Spain}},
  \bibfield{editor}{\bibinfo{person}{John Shawe{-}Taylor},
  \bibinfo{person}{Richard~S. Zemel}, \bibinfo{person}{Peter~L. Bartlett},
  \bibinfo{person}{Fernando C.~N. Pereira}, {and} \bibinfo{person}{Kilian~Q.
  Weinberger}} (Eds.). \bibinfo{pages}{1233--1241}.
\newblock


\bibitem[\protect\citeauthoryear{Khabiri, Gifford, Vinzamuri, Patel, and
  Mazzoleni}{Khabiri et~al\mbox{.}}{2019}]%
        {khabiri2019industry}
\bibfield{author}{\bibinfo{person}{Elham Khabiri}, \bibinfo{person}{Wesley~M
  Gifford}, \bibinfo{person}{Bhanukiran Vinzamuri}, \bibinfo{person}{Dhaval
  Patel}, {and} \bibinfo{person}{Pietro Mazzoleni}.}
  \bibinfo{year}{2019}\natexlab{}.
\newblock \showarticletitle{Industry specific word embedding and its
  application in log classification}. In \bibinfo{booktitle}{\emph{CIKM '19}}.
  \bibinfo{pages}{2713--2721}.
\newblock


\bibitem[\protect\citeauthoryear{Kim}{Kim}{2014}]%
        {kim-2014-convolutional}
\bibfield{author}{\bibinfo{person}{Yoon Kim}.} \bibinfo{year}{2014}\natexlab{}.
\newblock \showarticletitle{Convolutional Neural Networks for Sentence
  Classification}. In \bibinfo{booktitle}{\emph{EMNLP '14}}.
\newblock


\bibitem[\protect\citeauthoryear{Levy and Goldberg}{Levy and Goldberg}{2014}]%
        {NIPS2014_feab05aa}
\bibfield{author}{\bibinfo{person}{Omer Levy} {and} \bibinfo{person}{Yoav
  Goldberg}.} \bibinfo{year}{2014}\natexlab{}.
\newblock \showarticletitle{Neural Word Embedding as Implicit Matrix
  Factorization}. In \bibinfo{booktitle}{\emph{NeurIPS '14}}.
\newblock


\bibitem[\protect\citeauthoryear{Liu, Huang, Gao, Wei, Tian, and Liu}{Liu
  et~al\mbox{.}}{2018}]%
        {liu-etal-2018-task}
\bibfield{author}{\bibinfo{person}{Qian Liu}, \bibinfo{person}{Heyan Huang},
  \bibinfo{person}{Yang Gao}, \bibinfo{person}{Xiaochi Wei},
  \bibinfo{person}{Yuxin Tian}, {and} \bibinfo{person}{Luyang Liu}.}
  \bibinfo{year}{2018}\natexlab{}.
\newblock \showarticletitle{Task-oriented Word Embedding for Text
  Classification}. In \bibinfo{booktitle}{\emph{Proceedings of the 27th
  International Conference on Computational Linguistics}}.
  \bibinfo{publisher}{Association for Computational Linguistics},
  \bibinfo{address}{Santa Fe, New Mexico, USA}, \bibinfo{pages}{2023--2032}.
\newblock


\bibitem[\protect\citeauthoryear{Ma and Hovy}{Ma and Hovy}{2016}]%
        {ma-hovy-2016-end}
\bibfield{author}{\bibinfo{person}{Xuezhe Ma} {and} \bibinfo{person}{Eduard
  Hovy}.} \bibinfo{year}{2016}\natexlab{}.
\newblock \showarticletitle{End-to-end Sequence Labeling via Bi-directional
  {LSTM}-{CNN}s-{CRF}}. In \bibinfo{booktitle}{\emph{ACL '16}}.
\newblock


\bibitem[\protect\citeauthoryear{Maruf and Haffari}{Maruf and Haffari}{2018}]%
        {maruf-haffari-2018-document}
\bibfield{author}{\bibinfo{person}{Sameen Maruf} {and}
  \bibinfo{person}{Gholamreza Haffari}.} \bibinfo{year}{2018}\natexlab{}.
\newblock \showarticletitle{Document Context Neural Machine Translation with
  Memory Networks}. In \bibinfo{booktitle}{\emph{Proceedings of the 56th Annual
  Meeting of the Association for Computational Linguistics (Volume 1: Long
  Papers)}}. \bibinfo{publisher}{Association for Computational Linguistics},
  \bibinfo{address}{Melbourne, Australia}, \bibinfo{pages}{1275--1284}.
\newblock


\bibitem[\protect\citeauthoryear{Mikolov, Chen, Corrado, and Dean}{Mikolov
  et~al\mbox{.}}{2013a}]%
        {mikolov13}
\bibfield{author}{\bibinfo{person}{Tomas Mikolov}, \bibinfo{person}{Kai Chen},
  \bibinfo{person}{Greg~S. Corrado}, {and} \bibinfo{person}{Jeffrey Dean}.}
  \bibinfo{year}{2013}\natexlab{a}.
\newblock \showarticletitle{Efficient Estimation of Word Representations in
  Vector Space}.
\newblock \bibinfo{journal}{\emph{ICLR Workshop’13}} (\bibinfo{year}{2013}).
\newblock


\bibitem[\protect\citeauthoryear{Mikolov, Sutskever, Chen, Corrado, and
  Dean}{Mikolov et~al\mbox{.}}{2013b}]%
        {DBLP:conf/nips/MikolovSCCD13}
\bibfield{author}{\bibinfo{person}{Tom{\'{a}}s Mikolov}, \bibinfo{person}{Ilya
  Sutskever}, \bibinfo{person}{Kai Chen}, \bibinfo{person}{Gregory~S. Corrado},
  {and} \bibinfo{person}{Jeffrey Dean}.} \bibinfo{year}{2013}\natexlab{b}.
\newblock \showarticletitle{Distributed Representations of Words and Phrases
  and their Compositionality}. In \bibinfo{booktitle}{\emph{NeurIPS '13}},
  \bibfield{editor}{\bibinfo{person}{Christopher J.~C. Burges},
  \bibinfo{person}{L{\'{e}}on Bottou}, \bibinfo{person}{Zoubin Ghahramani},
  {and} \bibinfo{person}{Kilian~Q. Weinberger}} (Eds.).
\newblock


\bibitem[\protect\citeauthoryear{Mnih and Kavukcuoglu}{Mnih and
  Kavukcuoglu}{2013}]%
        {Mnih}
\bibfield{author}{\bibinfo{person}{Andriy Mnih} {and} \bibinfo{person}{Koray
  Kavukcuoglu}.} \bibinfo{year}{2013}\natexlab{}.
\newblock \showarticletitle{Learning word embeddings efficiently with
  noise-contrastive estimation}. In \bibinfo{booktitle}{\emph{NeurIPS '13}},
  \bibfield{editor}{\bibinfo{person}{C.~J.~C. Burges},
  \bibinfo{person}{L.~Bottou}, \bibinfo{person}{M.~Welling},
  \bibinfo{person}{Z.~Ghahramani}, {and} \bibinfo{person}{K.~Q. Weinberger}}
  (Eds.).
\newblock


\bibitem[\protect\citeauthoryear{Ortiz~Su{\'a}rez, Romary, and
  Sagot}{Ortiz~Su{\'a}rez et~al\mbox{.}}{2020}]%
        {ortiz-suarez-etal-2020-monolingual}
\bibfield{author}{\bibinfo{person}{Pedro~Javier Ortiz~Su{\'a}rez},
  \bibinfo{person}{Laurent Romary}, {and} \bibinfo{person}{Beno{\^\i}t Sagot}.}
  \bibinfo{year}{2020}\natexlab{}.
\newblock \showarticletitle{A Monolingual Approach to Contextualized Word
  Embeddings for Mid-Resource Languages}. In \bibinfo{booktitle}{\emph{ACL
  '20}}.
\newblock


\bibitem[\protect\citeauthoryear{Pennington, Socher, and Manning}{Pennington
  et~al\mbox{.}}{2014}]%
        {DBLP:conf/emnlp/PenningtonSM14}
\bibfield{author}{\bibinfo{person}{Jeffrey Pennington},
  \bibinfo{person}{Richard Socher}, {and} \bibinfo{person}{Christopher~D
  Manning}.} \bibinfo{year}{2014}\natexlab{}.
\newblock \showarticletitle{Glove: Global vectors for word representation}. In
  \bibinfo{booktitle}{\emph{EMNLP '14}}. \bibinfo{pages}{1532--1543}.
\newblock


\bibitem[\protect\citeauthoryear{Peters, Neumann, Iyyer, Gardner, Clark, Lee,
  and Zettlemoyer}{Peters et~al\mbox{.}}{2018}]%
        {Elmo}
\bibfield{author}{\bibinfo{person}{Matthew Peters}, \bibinfo{person}{Mark
  Neumann}, \bibinfo{person}{Mohit Iyyer}, \bibinfo{person}{Matt Gardner},
  \bibinfo{person}{Christopher Clark}, \bibinfo{person}{Kenton Lee}, {and}
  \bibinfo{person}{Luke Zettlemoyer}.} \bibinfo{year}{2018}\natexlab{}.
\newblock \showarticletitle{Deep Contextualized Word Representations}. In
  \bibinfo{booktitle}{\emph{Proceedings of NAACL-HLT)}}.
  \bibinfo{publisher}{Association for Computational Linguistics},
  \bibinfo{address}{New Orleans, Louisiana}, \bibinfo{pages}{2227--2237}.
\newblock


\bibitem[\protect\citeauthoryear{Rebentrost, Mohseni, and Lloyd}{Rebentrost
  et~al\mbox{.}}{2014}]%
        {rebentrost2014quantum}
\bibfield{author}{\bibinfo{person}{Patrick Rebentrost}, \bibinfo{person}{Masoud
  Mohseni}, {and} \bibinfo{person}{Seth Lloyd}.}
  \bibinfo{year}{2014}\natexlab{}.
\newblock \showarticletitle{Quantum support vector machine for big data
  classification}.
\newblock \bibinfo{journal}{\emph{Physical review letters}}
  \bibinfo{volume}{113}, \bibinfo{number}{13} (\bibinfo{year}{2014}),
  \bibinfo{pages}{130503}.
\newblock


\bibitem[\protect\citeauthoryear{Socher, Perelygin, Wu, Chuang, Manning, Ng,
  and Potts}{Socher et~al\mbox{.}}{2013}]%
        {SST2}
\bibfield{author}{\bibinfo{person}{Richard Socher}, \bibinfo{person}{Alex
  Perelygin}, \bibinfo{person}{Jean Wu}, \bibinfo{person}{Jason Chuang},
  \bibinfo{person}{Christopher~D Manning}, \bibinfo{person}{Andrew~Y Ng}, {and}
  \bibinfo{person}{Christopher Potts}.} \bibinfo{year}{2013}\natexlab{}.
\newblock \showarticletitle{Recursive deep models for semantic compositionality
  over a sentiment treebank}. In \bibinfo{booktitle}{\emph{EMNLP '13}}.
  \bibinfo{pages}{1631--1642}.
\newblock


\bibitem[\protect\citeauthoryear{Song, Woodruff, and Zhang}{Song
  et~al\mbox{.}}{2016}]%
        {DBLP:conf/nips/SongWZ16}
\bibfield{author}{\bibinfo{person}{Zhao Song}, \bibinfo{person}{David~P.
  Woodruff}, {and} \bibinfo{person}{Huan Zhang}.}
  \bibinfo{year}{2016}\natexlab{}.
\newblock \showarticletitle{Sublinear Time Orthogonal Tensor Decomposition}. In
  \bibinfo{booktitle}{\emph{Advances in Neural Information Processing Systems
  29: Annual Conference on Neural Information Processing Systems 2016, December
  5-10, 2016, Barcelona, Spain}}, \bibfield{editor}{\bibinfo{person}{Daniel~D.
  Lee}, \bibinfo{person}{Masashi Sugiyama}, \bibinfo{person}{Ulrike von
  Luxburg}, \bibinfo{person}{Isabelle Guyon}, {and} \bibinfo{person}{Roman
  Garnett}} (Eds.). \bibinfo{pages}{793--801}.
\newblock


\bibitem[\protect\citeauthoryear{Strubell, Ganesh, and McCallum}{Strubell
  et~al\mbox{.}}{2019}]%
        {strubell-etal-2019-energy}
\bibfield{author}{\bibinfo{person}{Emma Strubell}, \bibinfo{person}{Ananya
  Ganesh}, {and} \bibinfo{person}{Andrew McCallum}.}
  \bibinfo{year}{2019}\natexlab{}.
\newblock \showarticletitle{Energy and Policy Considerations for Deep Learning
  in {NLP}}. In \bibinfo{booktitle}{\emph{ACL '19}}.
  \bibinfo{publisher}{Association for Computational Linguistics},
  \bibinfo{address}{Florence, Italy}, \bibinfo{pages}{3645--3650}.
\newblock


\bibitem[\protect\citeauthoryear{Sun, Yu, Song, Liu, Yang, and Zhou}{Sun
  et~al\mbox{.}}{2020}]%
        {sun-etal-2020-mobilebert}
\bibfield{author}{\bibinfo{person}{Zhiqing Sun}, \bibinfo{person}{Hongkun Yu},
  \bibinfo{person}{Xiaodan Song}, \bibinfo{person}{Renjie Liu},
  \bibinfo{person}{Yiming Yang}, {and} \bibinfo{person}{Denny Zhou}.}
  \bibinfo{year}{2020}\natexlab{}.
\newblock \showarticletitle{{M}obile{BERT}: a Compact Task-Agnostic {BERT} for
  Resource-Limited Devices}. In \bibinfo{booktitle}{\emph{ACL '20}}.
  \bibinfo{publisher}{Association for Computational Linguistics},
  \bibinfo{address}{Online}, \bibinfo{pages}{2158--2170}.
\newblock


\bibitem[\protect\citeauthoryear{Tang}{Tang}{2019}]%
        {Ewin}
\bibfield{author}{\bibinfo{person}{Ewin Tang}.}
  \bibinfo{year}{2019}\natexlab{}.
\newblock \showarticletitle{A quantum-inspired classical algorithm for
  recommendation systems}.
\newblock \bibinfo{journal}{\emph{STOC '19}} (\bibinfo{year}{2019}).
\newblock


\bibitem[\protect\citeauthoryear{Tjong Kim~Sang and De~Meulder}{Tjong Kim~Sang
  and De~Meulder}{2003}]%
        {tjong-kim-sang-de-meulder-2003-introduction}
\bibfield{author}{\bibinfo{person}{Erik~F. Tjong Kim~Sang} {and}
  \bibinfo{person}{Fien De~Meulder}.} \bibinfo{year}{2003}\natexlab{}.
\newblock \showarticletitle{Introduction to the {C}o{NLL}-2003 Shared Task:
  Language-Independent Named Entity Recognition}. In
  \bibinfo{booktitle}{\emph{NAACL '03}}.
\newblock


\bibitem[\protect\citeauthoryear{Tropp}{Tropp}{2015}]%
        {tropp}
\bibfield{author}{\bibinfo{person}{Joel~A. Tropp}.}
  \bibinfo{year}{2015}\natexlab{}.
\newblock \showarticletitle{An Introduction to Matrix Concentration
  Inequalities}.
\newblock  (\bibinfo{year}{2015}).
\newblock


\bibitem[\protect\citeauthoryear{Turney and Pantel}{Turney and Pantel}{2010}]%
        {turney}
\bibfield{author}{\bibinfo{person}{Peter~D. Turney} {and}
  \bibinfo{person}{Patrick Pantel}.} \bibinfo{year}{2010}\natexlab{}.
\newblock \showarticletitle{From Frequency to Meaning: Vector Space Models of
  Semantics}.
\newblock \bibinfo{journal}{\emph{J. Artif. Int. Res.}} (\bibinfo{year}{2010}).
\newblock


\bibitem[\protect\citeauthoryear{Zeng and Coecke}{Zeng and Coecke}{2016}]%
        {zeng2016quantum}
\bibfield{author}{\bibinfo{person}{William Zeng} {and} \bibinfo{person}{Bob
  Coecke}.} \bibinfo{year}{2016}\natexlab{}.
\newblock \showarticletitle{Quantum algorithms for compositional natural
  language processing}.
\newblock \bibinfo{journal}{\emph{arXiv preprint arXiv:1608.01406}}
  (\bibinfo{year}{2016}).
\newblock


\bibitem[\protect\citeauthoryear{Zhang, Liu, Zhu, Zheng, Liu, Wang, Chen, and
  Zhai}{Zhang et~al\mbox{.}}{2019}]%
        {zhang2019learning}
\bibfield{author}{\bibinfo{person}{Yun Zhang}, \bibinfo{person}{Yongguo Liu},
  \bibinfo{person}{Jiajing Zhu}, \bibinfo{person}{Ziqiang Zheng},
  \bibinfo{person}{Xiaofeng Liu}, \bibinfo{person}{Weiguang Wang},
  \bibinfo{person}{Zijie Chen}, {and} \bibinfo{person}{Shuangqing Zhai}.}
  \bibinfo{year}{2019}\natexlab{}.
\newblock \showarticletitle{Learning Chinese word embeddings from stroke,
  structure and pinyin of characters}. In \bibinfo{booktitle}{\emph{CIKM '19}}.
  \bibinfo{pages}{1011--1020}.
\newblock


\end{thebibliography}


\end{document}